\documentclass{article}

\usepackage[main,final]{neurips_2026}

\usepackage[utf8]{inputenc}
\usepackage[T1]{fontenc}
\usepackage{hyperref}
\usepackage{url}
\usepackage{microtype}
\usepackage{graphicx}
\usepackage{booktabs}
\usepackage{subcaption}
\usepackage{nicefrac}
\usepackage{xcolor}
\usepackage[most]{tcolorbox}
\usepackage{wrapfig}

\usepackage{amsmath}
\usepackage{amssymb}
\usepackage{amsfonts}
\usepackage{mathtools}
\usepackage{amsthm}
\newtheorem{theorem}{Theorem}[section]
\newtheorem{proposition}[theorem]{Proposition}
\newtheorem{lemma}[theorem]{Lemma}
\newtheorem{corollary}[theorem]{Corollary}

\newtheorem{assumption}[theorem]{Assumption}

\usepackage{thm-restate}
\DeclareMathOperator\supp{supp}
\DeclareMathOperator\Top{Top}

\usepackage[capitalize,noabbrev]{cleveref}

\newtcolorbox{takeaway}[1]{colback=orange!5!white,
  colframe=orange!75!black,
  fonttitle=\bfseries,
  title=#1
}

\usepackage[disable,textsize=tiny]{todonotes}

\title{AC/DC on a Budget -- Alternating Sparse Phases}

\author{Rahul Nittala$^{1}$, Advait Gadhikar$^{1}$, Tom Jacobs$^{1}$, Rebekka Burkholz$^{1}$ \\
  $^{1}$ CISPA Helmholtz Center for Information Security, Germany\\
  \texttt{rahul.nittala@cispa.de}, \texttt{advait.gadhikar@cispa.de}, \texttt{tom.jacobs@cispa.de}, \\ \texttt{burkholz@cispa.de} \\
}

\begin{document}

\maketitle

\begin{abstract}

State-of-the-art neural network sparsification relies on alternating dense and sparse training phases. While computationally expensive, dense phases are critical for performance. Our empirical analysis reveals the mechanics behind this gain: upon reintroduction, previously masked parameters receive disproportionately large gradients that drive essential mask rewiring and parameter sign-flipping. Building on these insights, we optimize the performance-to-FLOPs Pareto frontier through three innovations: (1) zero-initializing masked parameters during dense phases to promote relevant sign flips, (2) shortening dense phases, and (3) replacing the dense phase entirely with a sparse rewiring phase that updates only the active mask and a small, high-gradient fraction of out-of-mask weights. Evaluated on CIFAR100 and ImageNet, our variant ARC (Alternating Rewiring Compression) matches baseline AC/DC performance at a fraction of the computational cost, establishing a superior performance-to-FLOPs frontier also compared to standard sparse-to-sparse algorithms. 
Experiments with GPT-2 on OpenWebText further show improvements
in validation perplexity and average zero-shot performance.
\end{abstract}

\section{Introduction}
The paradigm of overparameterization has led to significant breakthroughs across various machine learning tasks by enhancing optimization landscapes and model generalization (\citep{neyshabur2017exploring}, \citep{pmlr-v119-sankararaman20a}, \citep{barrett2021implicit}, \citep{belkin2019reconciling}). Nevertheless, the high computational costs associated with training and deploying these models have increased the need for sparsification methods to alleviate hardware constraints (\citep{liu2021we}, \citep{cheng2024survey}). However, most state-of-the-art sparsification methods, ranging from iterative magnitude pruning \citep{han-efficient-nns, rewindVsFinetune} to continuous reparameterization \citep{kusupati20, ziyin2023spredsolvingl1penalty}, rely on the dense model to guide the search for a generalizable sparse structure. While Pruning at Initialization (PaI) attempts to bypass this by training sparse masks from scratch, it often fails to match the performance of its dense-to-sparse counterparts \citep{frankle2021review}. This suggests that the initial dense training phase—or the analogous exploration in Dynamic Sparse Training (DST) \citep{mocanu-adaptive-sparse-train, evci-rigl} through dense gradients—provides a critical "alignment" or "support expansion" that is currently irreplaceable.
In this paper, we seek to understand the precise role dense training plays in finding optimal sparse masks and ask: Can this computational cost be minimized without sacrificing the benefits of dense alignment?

To investigate this, we use the AC/DC algorithm (Alternating Compressed/DeCompressed training) \citep{peste2021acdc} as our primary lens. AC/DC’s explicit alternation between dense and sparse phases allows us to isolate and diagnose the impact of dense updates. We frame this investigation through the theoretical lens of saddle-to-saddle dynamics \citep{pesme2023saddletosaddledynamicsdiagonallinear, zhang2026saddle}. Under this view, training is a sequence of transitions between singular modes; we argue that dense phases act as "rewiring windows" that provide access to off-support directions necessary to escape current sparse saddles.
Moreover, we characterize this mechanism empirically by measuring parameter dynamics—specifically sign-flips and activation subspace stability. We find that dense phases facilitate critical alignment by allowing weights with small magnitudes (typically masked) but large gradients to re-enter the mask and correct their directional signal \citep{gadhikar2025signin, zhou2025pay}. 
Building on these insights, we make three propositions: (i) zero-initializing masked parameters during dense phases, (ii) truncating dense phases, (iii) Alternating Rewiring Compression (ARC), a method that replaces full dense updates with sparse updates (active sparse mask and a fraction of out of mask weights with large gradients). These three modifications each target a different aspect of the dense phase cost, while preserving the support-expansion mechanism we identify as essential.
Our contributions are as follows:
\begin{itemize}
    \item We propose a theory-motivated view of alternating sparse training as a sequence of sparse consolidation followed by support expansion, which links compressed-gradient Iterative Hard Thresholding (IHT) and saddle-to-saddle dynamics.

    \item We empirically show that sparse phases refine the current support, while dense phases expose high-gradient off-support coordinates that drive mask rewiring and sign corrections.

    \item We introduce three innovations: Zero init, truncated dense phases, and the method ARC, which reduces the dense-phase cost by updating only a small high-gradient subset of off-support parameters.

    \item We demonstrate on CIFAR100 and ImageNet that ARC matches or exceeds the performance of AC/DC while significantly reducing theoretical FLOPs, improving the performance-to-FLOP tradeoff compared to AC/DC and RigL.
    GPT-2 experiments on OpenWebText additionally demonstrate
    improved perplexity and average zero-shot performance.
\end{itemize}

\subsection{Related work}

\paragraph{Existence of sparse solutions.}
The feasibility of pruning neural networks has been popularized by the lottery ticket \citep{frankle2019lottery}, which has also been proven theoretically \citep{malach2020proving,pensia2020optimal,orseau2020logarithmic,uniExist,depthexist,cnnexist,er-paper,ferbach2023a}.  
Related experimental studies and theoretical investigations suggest that iterative approaches that learn the sparse structure and corresponding weights simultaneously might be required \citep{paul2023unmasking, gadhikar2024masks}.

\paragraph{Sparse training.}
Sparsification methods aim to find a sparse network that can generalize well with a subset of active parameters.
This strategy description covers a broad class of algorithms.
Pruning methods like Iterative Magnitude Pruning \citep{frankle2019lottery, han-efficient-nns}, Learning Rate Rewinding \citep{rewindVsFinetune}, Woodfisher \citep{singh2020woodfisher}, CAP \citep{kuznedelev2024cap}, Optimal Brain Surgeon \citep{optimal-brain-surgeon, combinatorial-brain-surgeon} start training a dense network and gradually sparsify it during the course of training, eliminating parameters after scoring them based on magnitude or other importance metrics. The post-training and iterative dense-to-sparse methods benefit from dense optimization and gradual support refinement, but retain substantial dense training cost. Continuous sparsification methods like spred \citep{ziyin2023spredsolvingl1penalty}, PILoT \citep{jacobs2024maskmirrorimplicitsparsification}, STR\citep{kusupati20}, L0 regularized training \citep{louizos2018learning} reparameterize training or use a regularizer to transition from dense to sparse networks.
Even deeper reparameterization have been utilized for unstructured and structured sparse training \citep{kolb2025deep, Kolb2025DifferentiableSV}.
An orthogonal line of work, dynamic sparse training (DST) trains a random sparse network instead of a dense one. 
During training, DST methods smartly explore the space of the dense network by rewiring the mask randomly \citep{mocanu-adaptive-sparse-train}, based on gradient information \cite{evci-rigl, struct-dst} or graph connectivity \cite{chts}. As DST methods instead keep the model sparse while periodically changing its connectivity, it can reduce theoretical training FLOPs but makes support discovery more constrained. 

AC/DC \citep{peste2021acdc} lies between dense-to-sparse and
sparse-training approaches: it alternates fully dense exploration
with fixed-support sparse consolidation.
This provides access to dense optimization for support discovery,
but incurs the dense FLOPs cost at every decompression.
ARC retains the dense warmup and alternating support refinement,
while replacing subsequent dense phases with sparser expansions of the inherited sparse support selected based on gradients.
Our focus is therefore on retaining the support-expansion benefit
of AC/DC without the fully dense intermediate training phases. Top-KAST similarly explores parameters beyond its sparse forward support, using a larger backward support, both selected by weight magnitude \citep{jayakumar2020topkast}. In contrast, ARC selects additional parameters based on off-support gradients at rewiring-phase boundaries and trains the expanded support in both forward and backward passes before pruning back to the target sparsity.

\paragraph{Training a fixed sparse mask from scratch.}
Pruning at initialization (PaI) selects a sparse mask before
training and keeps its connectivity fixed
\citep{random-pruning-vita,er-paper,synflow,snip,grasp,
pham2023paths,patil2021phew}.
Its difficulty in matching dense-to-sparse methods
\citep{frankle2021review} motivates understanding what dense
training contributes to support discovery.
AC/DC's explicitly separated phases allow us to isolate this
contribution.

\paragraph{Small weights and sign flip importance.}
Recent works in sparse training
\citep{gadhikar2024masks, gadhikar2025signin} have highlighted sign flipping as a primary performance bottleneck, enabling relevant sign-flips improves sparse training and can also improve optimization \cite{jacobs2025ham}.
For this reason we define and measure these consequential sign-flips in the context of AC/DC to investigate parameter dynamics in dense training.
On top of that to encourage sign flips we can focus on small weight parameters in the dense/rewiring phase \citep{zhou2025pay}.

\paragraph{Implicit bias and saddle-to-saddle dynamics.}
Implicit bias studies how optimization, architecture, and parameterization select particular solutions in overparameterized models beyond what is specified by the training objective alone \citep{gunasekar2017implicit, woodworth2020kernel, pesme2021implicit, jacobs2024maskmirrorimplicitsparsification, jacobs2025mirrormirrorflowdoes, Li2022ImplicitBO, tom2026never, domin2024from, Wang2024AMD, Even2023SGDOD, vaškevičius2019implicitregularizationoptimalsparse}. 
More specifically, we draw on the saddle-to-saddle view of learning dynamics, a key driver behind implicit regularization. 
\cite{saxe2014exact} derive exact gradient-flow dynamics for deep linear networks and show that learning proceeds progressively across singular modes, typically prioritizing dominant modes first in line with \citep{baldi1989neural,kawaguchi2016deep,laurent2018deep}. 
Building on this type of learning description, \cite{jacot2021saddle} shows small-initialization training as a sequence of transitions between saddle regions, where symmetry and sparsity influence which directions become active. 
For deep diagonal linear networks, \cite{pesme2023saddletosaddledynamicsdiagonallinear} characterize saddle-to-saddle dynamics exactly in which coordinates are activated progressively along the optimization path or deactivated. 
Finally, \cite{zhang2026saddle} connect such transitions to simplicity bias in deep neural networks.
This perspective is particularly relevant to alternating sparse training in general: fixed sparse phases can be interpreted as consolidating the currently active representation, while dense or rewiring phases provide access to new off-support directions that enable the next transition. 
We use this viewpoint as conceptual motivation for studying mask rewiring and support expansion in AC/DC.

\section{Theory: Alternating sparse training phases}
\label{sec:Theory}

Our central question is whether the benefits of dense support
expansion can be retained without updating every parameter.
We first extend the convergence analysis of AC/DC to compressed
gradients, then examine how alternating sparse phases can
preserve saddle-to-saddle learning dynamics.

AC/DC (Alternating Compressed/DeCompressed Training)
\citep{peste2021acdc} follows an iterative hard-thresholding
(IHT) principle \citep{blumensath2008iterative} to minimize
$f:\mathbb{R}^N\to\mathbb{R}$ subject to $\|\theta\|_0\leq k$,
where $k$ is the target number of nonzero parameters.
After dense warmup, AC/DC alternates magnitude-pruned,
fixed-support sparse training with dense phases that restore
masked parameters to their values at the end of the previous
dense phase (\emph{Stale init}), followed by sparse finetuning
(\Cref{fig:acdc_schematic}).

\paragraph{Notation and compressed updates.}
For a vector $v$, $\Top_s(v)$ returns the indices of its $s$
largest-magnitude entries, with ties resolved by a fixed rule.
For a coordinate set $A$, $P_A$ retains coordinates in $A$
and zeros the others, while $v_A$ denotes restriction to $A$.
The hard-thresholding operator is $T_s(v)=P_{\Top_s(v)}v$.

At iteration $t$, let $\theta_t$ be the current parameters,
$S_t=\supp(\theta_t)$ their sparse support, and $g_t$ a
stochastic gradient.
We consider the expanded support $A_t=S_t\cup G_t$, where
$G_t=\Top_q((g_t)_{S_t^c})$ selects the $q$ largest-magnitude
gradient coordinates outside $S_t$.
The compressed gradient is $h_t=P_{A_t}g_t$.
The inherited support preserves the current sparse
representation, while $G_t$ provides additional directions
for support expansion.

\paragraph{Convergence with compressed gradients.}
We adapt the stochastic-IHT convergence analysis underlying
AC/DC by replacing the dense gradient with $h_t$.
As in the original analysis, we assume restricted smoothness,
which controls objective changes along sparse directions,
and a concentrated Polyak--{\L}ojasiewicz (CPL) condition,
which requires sufficiently large gradient energy to be
concentrated on a small set of coordinates.
For a global minimizer $\theta^\star$, the
$(s,\alpha)$-CPL condition requires
\begin{equation}
\label{eq:concentrated_pl_def}
\|T_s(\nabla f(\theta))\|_2^2
\geq 2\alpha\bigl(f(\theta)-f(\theta^\star)\bigr).
\end{equation}
The precise assumptions are provided in
Appendix~\ref{appendix : convergence}.

Replacing dense gradients with compressed updates additionally
requires that the discarded gradient information be sufficiently
small on the coordinates relevant to the IHT analysis.

\begin{assumption}[Compressed-gradient accuracy]
\label{ass:compressed-gradient-accuracy-main}
Let $\Delta_t=f(\theta_t)-f(\theta^\star)$.
For every support $U_t$ arising in the IHT analysis with
$|U_t|\leq 2k+q+3k^\star$, the compressed gradient satisfies
$\mathbb E[\|(h_t-\nabla f(\theta_t))_{U_t}\|_2^2
\mid\theta_t]\leq\sigma_{\rm cg}^2+\rho\Delta_t$,
for constants $\sigma_{\rm cg}^2,\rho\geq0$.
\end{assumption}

Here $\sigma_{\rm cg}^2$ bounds the error independent of
suboptimality, while $\rho$ controls the state-dependent
error introduced by gradient compression.

\begin{theorem}[Convergence with compressed gradients]
Let $f:\mathbb R^N\to\mathbb R$ have a $k^\star$-sparse
global minimizer $\theta^\star$.
Suppose $\beta>\alpha>0$, $f$ is
$(2k+q+3k^\star,\beta)$-restricted smooth and
$(k^\star,\alpha)$-CPL, and
Assumption~\ref{ass:compressed-gradient-accuracy-main} holds.
For universal constants $C_0,C_1>0$, let
$k\geq C_0k^\star(\beta/\alpha)^2$ and define
$\lambda_{\rm cg}=\alpha/(32\beta)-C_1\beta\rho$.
If $\lambda_{\rm cg}>0$, stochastic IHT with
$\eta=1/(2\beta)$ and update
$\theta_{t+1}=T_k(\theta_t-\eta h_t)$ converges
in $O(\lambda_{\rm cg}^{-1}\log(\Delta_0/\varepsilon))$
iterations to an expected suboptimality of at most
$\varepsilon+O(\beta\sigma_{\rm cg}^2/\lambda_{\rm cg})$.
\end{theorem}

The proof is provided in Appendix~\ref{appendix : convergence}.
Thus, compressed-gradient rewiring preserves the qualitative
IHT convergence behavior when the expanded support captures
sufficient dense-gradient information, with compression error
reducing the effective contraction rate.

\paragraph{Alternating sparse phases can preserve saddle-to-saddle
learning dynamics.}
Beyond convergence, we seek to understand why the coordinates
needed for learning change throughout training.
Saddle-to-saddle dynamics describe learning as successive
transitions between singular modes
\citep{pesme2023saddletosaddledynamicsdiagonallinear,zhang2026saddle}.
In deep linear networks, learning progressively increases
effective rank, prioritizing dominant singular directions
\citep{baldi1989neural,kawaguchi2016deep,laurent2018deep}.
We argue that alternating sparse phases can support the same
progression: sparse phases consolidate previously learned modes,
while rewiring phases expand the support to activate new ones.

We consider a linear network with one hidden layer,
$f_{W_1,W_2}(x)=W_2W_1x$ \citep{saxe2014exact},
trained with whitened inputs and population loss
$\mathcal L(W_1,W_2)=\frac12\|W_2W_1-M\|_F^2$.
Here $M=\mathbb E[yx^\top]
=U\operatorname{diag}(s_1,\ldots,s_r)V^\top$,
where $s_1>\cdots>s_r>0$ are the singular values,
with corresponding singular vectors $u_i,v_i$.
Write $\theta=(\operatorname{vec}(W_1),
\operatorname{vec}(W_2))$ and
$g(\theta)=\nabla_\theta\mathcal L(\theta)$.
Let $M_m=\sum_{i=1}^m s_i u_i v_i^\top$ denote
the first $m$ singular modes and $\mathcal M_m$ the
corresponding rank-$m$ saddle configurations.

We use $m$ to index successive mode-learning phases.
At phase $m$, let $\theta_m$ denote the parameters at
the phase boundary, $S_m$ the inherited sparse support,
and $A_m=S_m\cup G_m$ its expansion, where
$G_m=\Top_q((g(\theta_m))_{S_m^c})$.
Let $e_{m+1}(\theta)$ denote the dense gradient component
associated with activating the next singular mode.

The result relies on four assumptions: previously learned
modes admit sparse representations preserved by pruning;
the residual near $\mathcal M_m$ is dominated by the next
singular mode; the inherited support captures little of
the corresponding escape direction; and the expanded
support captures most of it.
We denote the sparse approximation error by $\delta_m$,
the fraction of the escape direction captured by $S_m$
by $\alpha_m$, and the fraction missed by $A_m$ by
$\varepsilon_m$.
The formal conditions are stated in
Assumptions~\ref{ass:sparse-current-mode}--\ref{ass:old-support-stall}
in Appendix~\ref{appendix : saddle dynamics}.
The connection to gradient compressibility is given in
Proposition~\ref{prop:gradient-compressibility}.

\begin{theorem}[Alternating sparse phases implement sparse saddle-to-saddle dynamics]
\label{thm:alternating-sparse-saddle-main}
Suppose Assumptions~\ref{ass:sparse-current-mode}--
\ref{ass:old-support-stall} hold for phases $m=0,\ldots,r-1$,
and $\mathcal L$ is $L$-smooth in a neighborhood of the
relevant trajectories.
Then each alternating cycle has the following behavior.
\begin{enumerate}
    \item \textbf{Consolidation.}
    Training on $S_m$ primarily optimizes the sparse
    representation of previously learned modes $1,\ldots,m$,
    approaching a neighborhood of the rank-$m$ saddle
    $\mathcal M_m$.

    \item \textbf{Stalling on the old support.}
    The inherited support contains only a small fraction
    $\alpha_m$ of the escape direction toward mode $m+1$.
    Hence, training exclusively on $S_m$ leaves the
    trajectory near the current sparse saddle.

    \item \textbf{Sparse escape by rewiring.}
    The expanded support $A_m=S_m\cup G_m$ captures at least
    a $1-\varepsilon_m^2$ fraction of the dense escape-direction
    energy, satisfying
    $\|P_{A_m}e_{m+1}(\theta)\|_2^2
    \geq(1-\varepsilon_m^2)\|e_{m+1}(\theta)\|_2^2$.
    Thus, rewiring provides access to directions needed
    to increase the learned strength of mode $m+1$.

    \item \textbf{Support transition.}
    After rewiring, magnitude pruning retains the coordinates
    needed to represent the newly learned mode, producing
    a support $S_{m+1}$ that represents $M_{m+1}$ up to
    approximation error $\delta_{m+1}$.
\end{enumerate}
\end{theorem}

The result connects alternating sparse consolidation and
support expansion to progressive mode learning.
When the expanded support captures the relevant escape
directions, sparse rewiring can preserve the transitions
enabled by dense phases without requiring all parameters
to be updated.
The formal analysis is provided in
Appendix~\ref{appendix : saddle dynamics}.

\section{The Mechanics of AC/DC: Diagnosis}
\label{sec:acdc_diagnosis}


The theory in Section~\ref{sec:Theory} shows that alternating sparse training can be viewed as a sequence of consolidation and support-expansion steps. We next ask whether related phase roles appear in practical AC/DC training: do fixed-support sparse phases mainly refine the current representation, while dense phases expose useful off-support directions? For this, we study AC/DC before introducing our algorithmic modifications.
These diagnostics clarify the role of dense phases and indicate which parts of them may be compressible.
Unless otherwise stated, we use ResNet18 on CIFAR100; training details are provided in Appendix~\ref{appendix : training details}.

\subsection{Sparse Phases Consolidate While Dense Phases Rewire}
\label{subsec:sparse_consolidate_dense_rewire}

\begin{figure*}[ht!]
    \centering
    \includegraphics[width=0.9\textwidth]{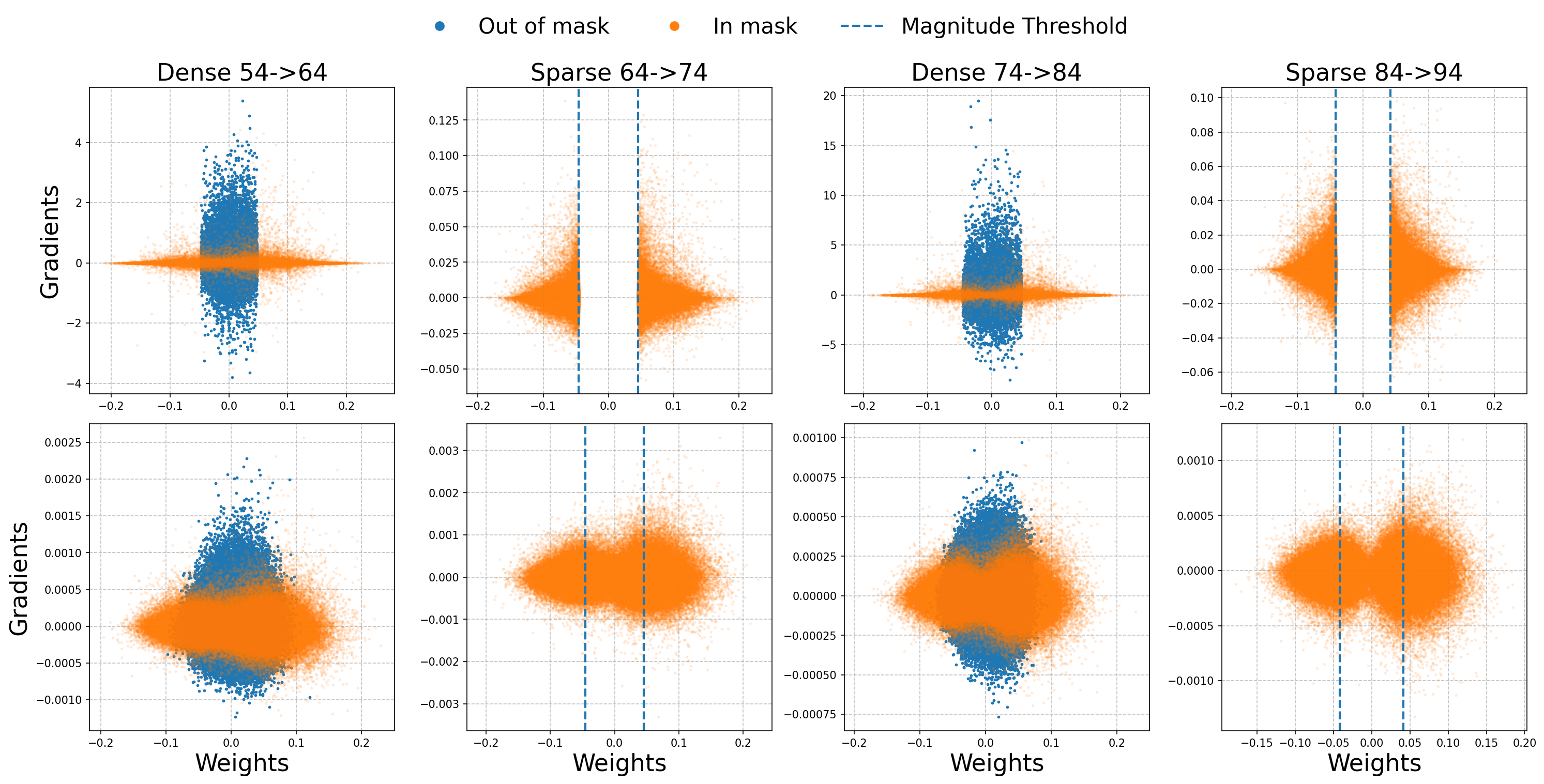}
    \caption{
    \textbf{Gradient--weight dynamics across AC/DC phases.}
    Gradient versus weight scatter plots for \textit{layer4.0.conv2} of a 90\% sparse ResNet18 on CIFAR100.
    Top and bottom rows show the beginning and end of each phase, respectively.
    During sparse phases, many active weights move toward small magnitude and small gradient.
    During dense phases, many reintroduced out-of-mask weights remain near zero but receive large gradients.
    }
    \label{fig:gradsVsWeights_0.1_10D_10S}
\end{figure*}

We first examine how coordinates retained by magnitude pruning
differ from those excluded by the current mask.
Figure~\ref{fig:gradsVsWeights_0.1_10D_10S} compares their weights
and gradients across consecutive phases, connecting mask
membership to the learning signal available to each coordinate.

\paragraph{Sparse phases refine the active support.}
Sparse training is restricted to the magnitude-pruned support
$S_t$.
Within this support, many active weights move toward smaller
magnitudes and also have small gradients.
Sparse phases therefore consolidate the current representation,
separating coordinates that carry useful signal from
low-magnitude, low-gradient coordinates that may be removed
at the next pruning step.

\paragraph{Dense phases expose off-support gradient signal.}
When masked parameters in $S_t^c$ are reintroduced, many have
smaller magnitudes but receive large gradients, forming the vertical band
near zero in Figure~\ref{fig:gradsVsWeights_0.1_10D_10S}.
This contrasts with the small, low-gradient active weights ($S_t$)
at the end of sparse training.
Dense phases allow these off-support coordinates to grow and
compete for inclusion in the next mask.
Thus, the dense phases act as rewiring windows, enabling the support for the next sparse phase to differ from the previous one. This mechanism is in line with the support-expansion role as described by
Theorem~\ref{thm:alternating-sparse-saddle-main}.

\subsection{Sign flips}
\label{subsection : sign flips}


The gradient--weight analysis suggests that dense phases
facilitate changes in the sparse mask. Yet, prior work has shown
that the signs of surviving weights are also important for
sparse training, and that improving sign alignment can help
close the gap between sparse training from scratch and
dense-to-sparse training
\citep{gadhikar2024masks,gadhikar2025signin}.
We therefore ask which phases of AC/DC establish the eventual
signs of the parameters retained in the final sparse model.
To study this retrospectively, we record the phase of each
final-mask parameter's last sign flip, after which the sign remains consistent across phases.
Parameters that never flip are assigned to the warmup bin.

Figure~\ref{fig:last_sign_flips_compare}a shows that most
final-mask parameters settle their signs by the end of warmup.
Among parameters corrected later, the last sign flips occur
predominantly during dense rather than sparse phases.
Dense phases therefore do not merely produce transient
gradient activity: they provide windows in which parameters
retained by the final sparse model resolve their eventual signs.


We compare two initialization strategies for masked parameters
at the beginning of each dense phase.
\emph{Stale init} restores these parameters to their values
at the end of the previous dense phase, whereas
\emph{Zero init} resets them to zero.
With Stale init, most late sign corrections occur in early
dense phases, while Zero init continues to produce corrections
in later phases.
This suggests that initialization influences when final
parameter signs become established, motivating our Zero init
modification in Section~\ref{sec:method}.
The detailed counting protocol, layerwise and shortened-phase
results, and corresponding ARC analysis are provided in
Appendix~\ref{appendix : sign flip rewire}.

\begin{figure*}[ht!]
    \centering
    \begin{subfigure}[b]{0.48\textwidth}
        \centering
        \includegraphics[width=\linewidth]{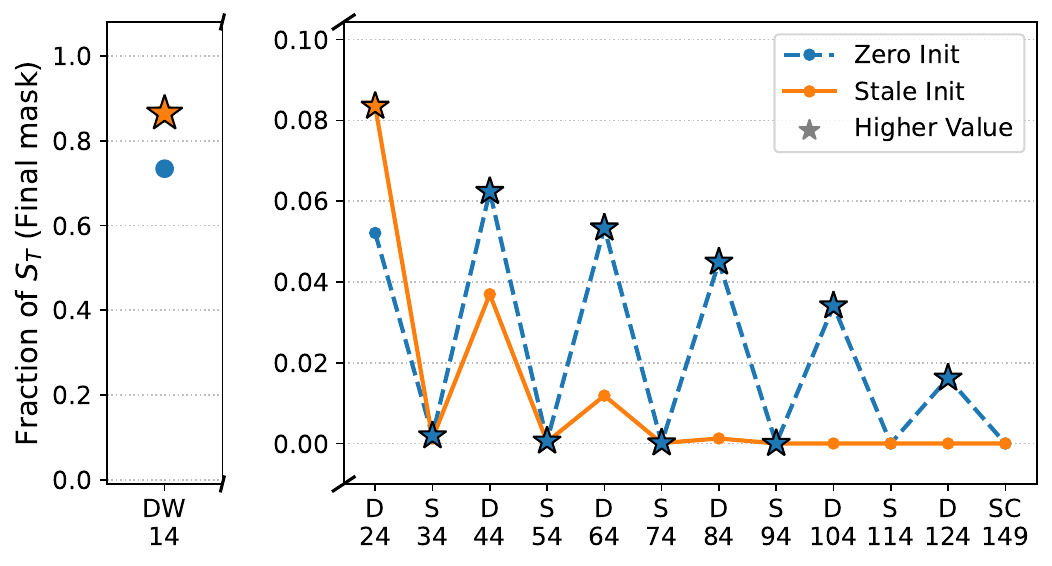}
        \caption{AC/DC}
        \label{fig:last_sign_flips_acdc}
    \end{subfigure}
    \hfill
    \begin{subfigure}[b]{0.48\textwidth}
        \centering
        \includegraphics[width=\linewidth]{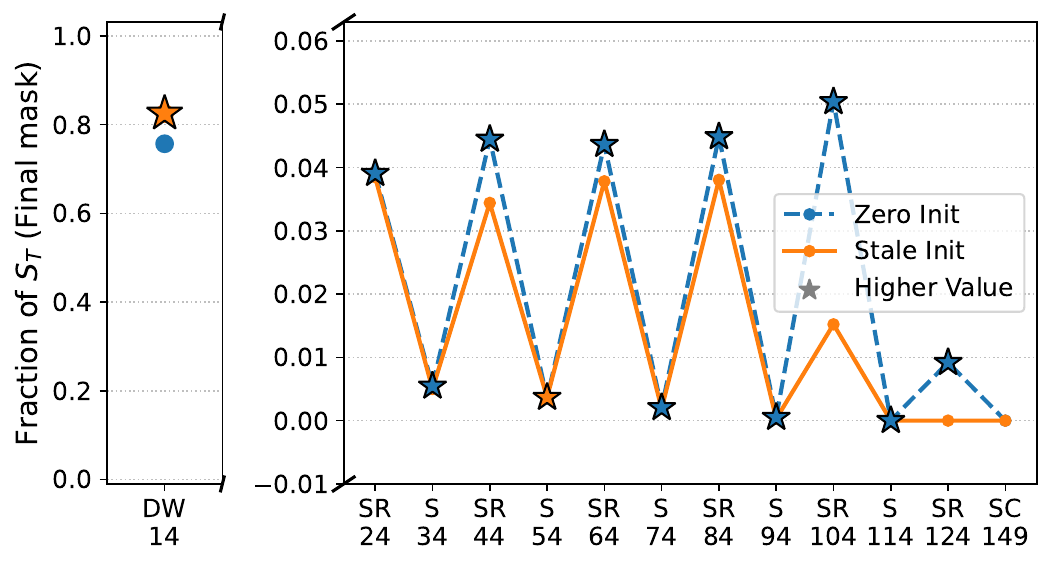}
        \caption{ARC}
        \label{fig:last_sign_flips_arc}
    \end{subfigure}

    \caption{
    \textbf{Timing of final sign alignment.}
    For each parameter retained in the final sparse mask, we record the phase in which it flips sign for the last time.
    The y-axis is normalized by the final mask size.
    DW, D, S, and SC denote dense warmup, dense, sparse, and sparse cooldown phases for AC/DC; in ARC, the corresponding non-sparse phases are rewiring phases on \(A_t=S_t\cup G_t\).
    \textbf{(a)} In AC/DC, most final-mask parameters align during dense warmup, while later sign corrections occur mainly during dense phases.
    \textbf{(b)} The analogous ARC diagnostic is discussed after ARC is introduced in \Cref{sec:method}.
    }
    \label{fig:last_sign_flips_compare}
\end{figure*}

\subsection{Activation Subspaces Reflect Sparse Consolidation and Dense Rewiring}
\label{sec:activation_subspace_stability}

\begin{figure*}[t]
    \centering
    \begin{subfigure}{0.48\textwidth}
        \centering
        \includegraphics[width=\textwidth]{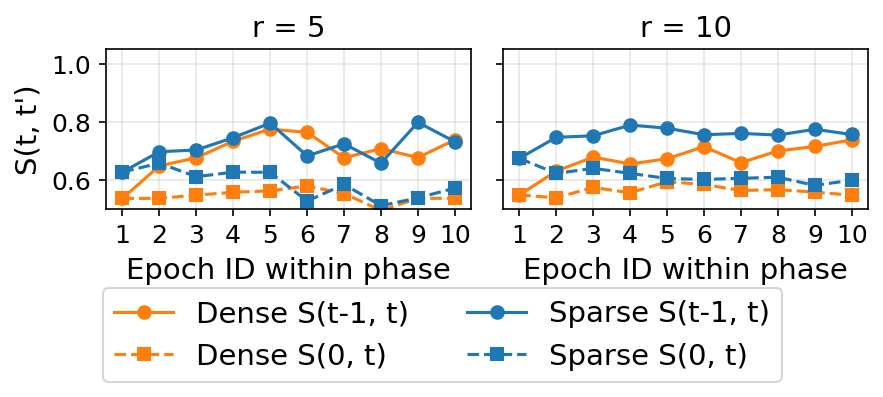}
        \caption{$1^{st}$: Sparse 24--34, Dense 34--44}
        \label{fig:activation_subspace_phase0}
    \end{subfigure}
    \hfill
    \begin{subfigure}{0.48\textwidth}
        \centering
        \includegraphics[width=\textwidth]{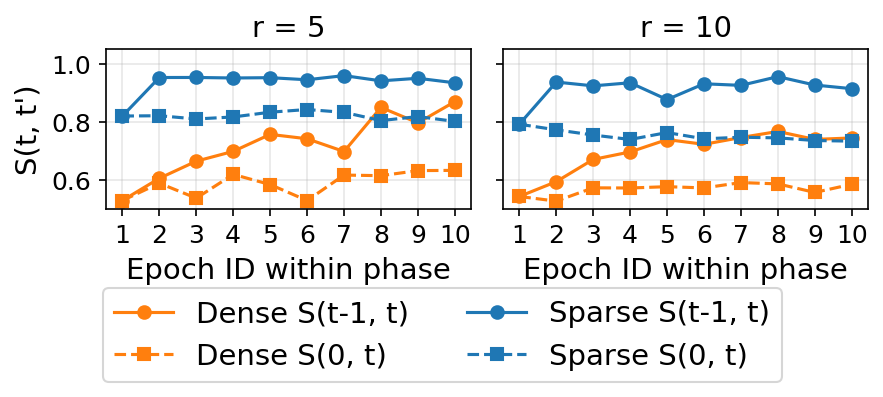}
        \caption{$4^{th}$ cycle: Sparse 84--94, Dense epochs 94--104}
        \label{fig:activation_subspace_phase1}
    \end{subfigure}

    \caption{
    \textbf{Sparse phases preserve activation subspaces.}
    Top-\(r\) activation Gram subspace similarity for \textit{layer3.1.conv2} of a 90\% sparse ResNet18 on CIFAR100.
    We show an early and later AC/DC cycle.
    Sparse phases maintain higher subspace similarity. Dense phases show larger but smooth drift.
    }
    \label{fig:activation_subspace_similarity}
\end{figure*}

The preceding diagnostics characterize AC/DC at the level of
individual parameters, gradients, and masks.
We next ask whether the alternating phases also leave a
representation-level signature.
The saddle-to-saddle perspective suggests that sparse phases
consolidate the current representation, while dense phases
allow larger but structured representation changes.
We therefore track the leading eigenspaces of activation
Gram matrices computed over the CIFAR100 training set
at each epoch checkpoint.
Their construction is detailed in
Appendix~\ref{appendix:activation_eigenspaces}.

Let $U_t^{(r)}$ denote an orthonormal basis for the subspace
spanned by the top $r$ activation Gram eigenvectors at epoch $t$.
We measure subspace overlap as
$
S_r(t_1,t_2)
=
\frac{1}{r}
\left\|
{U_{t_1}^{(r)}}^\top U_{t_2}^{(r)}
\right\|_F^2.
$
We report both $S_r(t-1,t)$, measuring consecutive-epoch
stability, and $S_r(0,t)$, measuring similarity to the first
checkpoint of each phase.

Figure~\ref{fig:activation_subspace_similarity} compares an
early and a later AC/DC cycle for \textit{layer3.1.conv2}.
Sparse phases preserve the leading activation subspaces more
strongly than dense phases, especially for $r=5$ and $r=10$.
This is consistent with sparse phases primarily consolidating
the current representation under a fixed support.
Dense phases, in contrast, show larger drift from the phase
start, indicating more substantial representation updates.
Importantly, dense-phase similarity remains substantial and
changes smoothly across epochs.
Thus, dense phases do not appear to reset the representation,
but rather modify it while preserving much of the dominant
activation subspace.
The distinction between the phases becomes clearer in later cycles, where
sparse phases show very high consecutive-epoch similarity
while dense phases continue to permit representation changes.
Full results across all four cycles and eigendirection heatmaps
are provided in
Appendix~\ref{appendix:activation_eigenspaces}.

Together, the diagnostics identify dense phases as rewiring
windows: they expose high-gradient off-support coordinates,
enable persistent sign corrections, and permit smooth
representation changes.
This motivates preserving the inherited support while replacing
full dense optimization with targeted updates of high-gradient
off-support coordinates.

\section{Method: Alternating Rewiring Compression}
\label{sec:method}


Motivated by these diagnostics, we pursue three complementary
modifications to improve the performance--FLOP tradeoff:
Zero init, truncated dense phases, and
\emph{Alternating Rewiring Compression} (ARC).


\paragraph{Zero init.}
In standard AC/DC, reintroduced parameters inherit their
previous dense-phase values (\emph{Stale init}), including
their signs.
Motivated by the sign-flip diagnostics in
Section~\ref{subsection : sign flips}, we instead consider
\emph{Zero init}, which resets these parameters to zero:
$\theta_i\leftarrow 0$ for $i\in S_t^c$.
This removes inherited sign information, allowing the current
gradient to determine their initial direction before they
compete for inclusion in the next sparse mask.
Zero init applies to all masked parameters during AC/DC's
dense phases, and only to the selected rewiring coordinates
$G_t$ in ARC.

\paragraph{Truncated dense phases.}
Standard AC/DC uses dense and sparse phases of equal length.
If dense phases serve as rewiring windows, an immediate practical question is how long those windows need to be.
We therefore evaluate schedules that keep the total cycle length fixed but shorten the dense portion.
If a standard cycle uses \(D\) dense epochs and \(S\) sparse epochs, we replace it with \(D'<D\) dense epochs and $S' = D+S-D'$
sparse epochs.
This preserves the alternating schedule while reducing theoretical FLOPs by replacing dense epochs with sparse epochs.
Furthermore, it is consistent with the smooth and gradual nature of the representation changes observed in Section \ref{sec:activation_subspace_stability}.

\paragraph{ARC: sparse rewiring instead of dense phases.}
Zero init and truncation reduce the cost or improve the behavior of dense phases, but still rely on dense updates.
ARC replaces the dense phase with a \emph{sparse rewiring phase}.
Instead of updating all parameters, ARC updates only $ A_t = S_t \cup G_t,$ with $G_t = \Top_q\!\left((|g_t|)_{S_t^c}\right)$,
where \(G_t\) contains the largest-gradient coordinates outside the current sparse support.
Here, $q$ is the number of added parameters that were previously masked.
In Section~\ref{sec:end_to_end_results}, the $d$ in label like $ARC (d)$ specifies
the density of $G_{t}$ over the set of all parameters.
The gradients used to construct \(G_t\) are aggregated over a small number of randomly sampled mini-batches, and \(G_t\) is selected once at the beginning of each rewiring phase. See Appendix ~\ref{appendix : training details} for more details.

During the rewiring phase, gradients and updates are restricted to \(A_t\).
The inherited support \(S_t\) continues training the current sparse representation, while \(G_t\) gives high-gradient off-support coordinates,  identified as the primary driver of mask rewiring in Section \ref{subsec:sparse_consolidate_dense_rewire}, an opportunity to grow, and compete for inclusion in the next sparse mask.
After the rewiring phase, magnitude pruning returns the model to the target sparsity, producing the sparse support for the following sparse phase.
ARC is therefore a sparse-to-sparse method after the initial dense warmup.
It keeps the alternating structure of AC/DC, but compresses the dense rewiring step onto \(A_t=S_t\cup G_t\).
Larger \(G_t\) more closely approximates full dense rewiring, while smaller \(G_t\) yields lower theoretical FLOPs.

\paragraph{Combined ARC setting.}
The three modifications can be applied together: our combined
setting uses Zero init, shortened rewiring windows, and sparse
updates on $A_t=S_t\cup G_t$.
Figure~\ref{fig:last_sign_flips_compare}b shows that ARC retains
the concentration of late sign corrections in rewiring phases.
Appendix~\ref{appendix : sign flip rewire} gives the full
comparison with AC/DC and the differences between the
initialization variants.

\section{End-to-End Performance and FLOP Tradeoffs}
\label{sec:end_to_end_results}

\begin{figure*}[t!]
    \centering
    \includegraphics[width=0.42\linewidth]{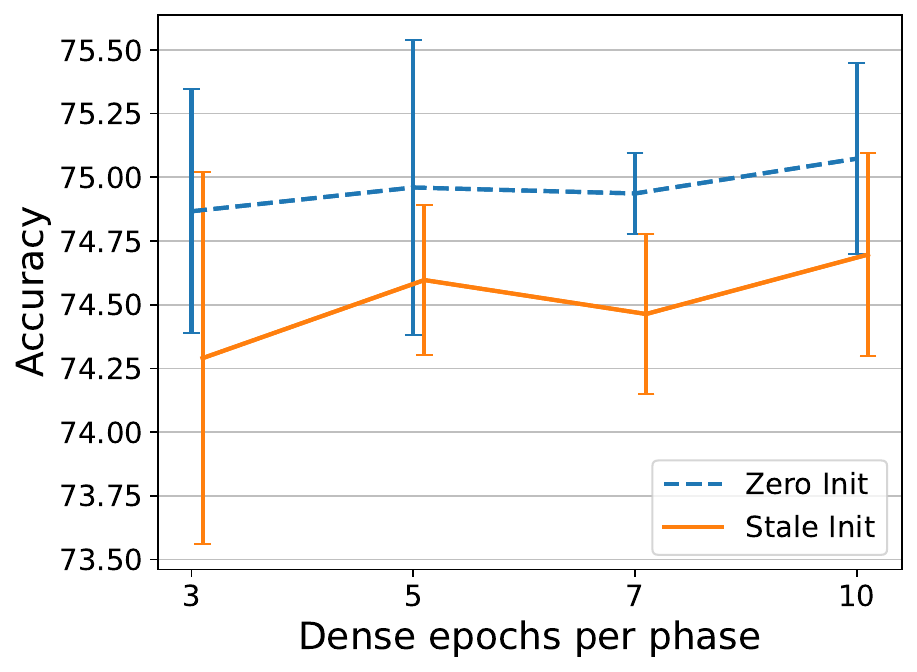}
    \includegraphics[width=0.42\linewidth]{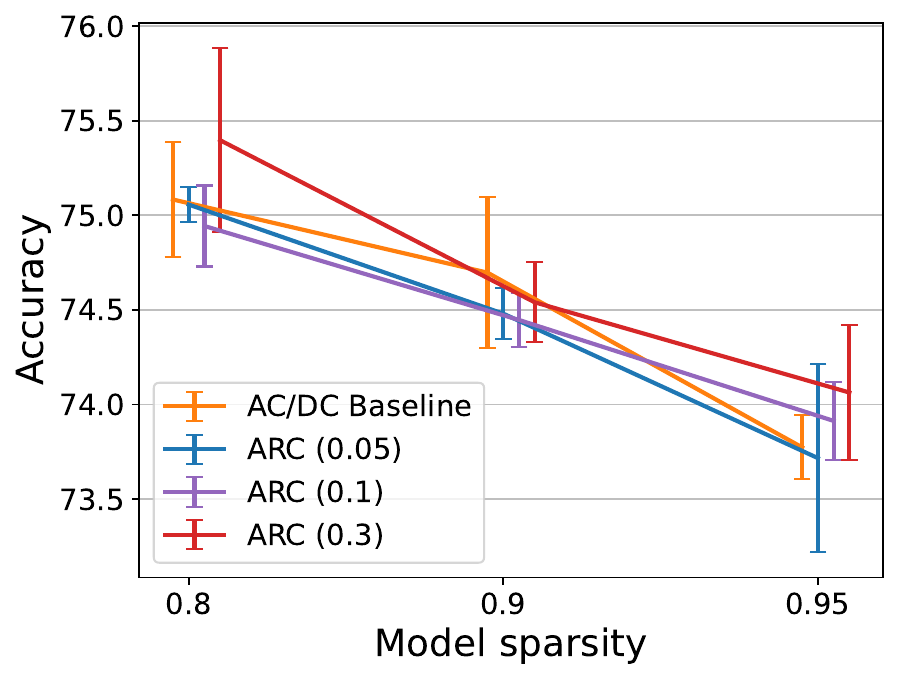}
    \caption{
\textbf{Reducing and sparsifying rewiring phases.}
Results for ResNet18 on CIFAR100.
Left: varying the number of dense epochs per AC/DC cycle
at 90\% sparsity.
Right: comparing AC/DC with ARC variants using different
$G_t$ budgets across target sparsities.
}
    \label{fig:pareto-dense-phases-c100}
    \label{fig:pareto-dense-grad-c100}
\end{figure*}

\begin{wrapfigure}{r}{0.52\linewidth}
    \centering
    \includegraphics[width=\linewidth]{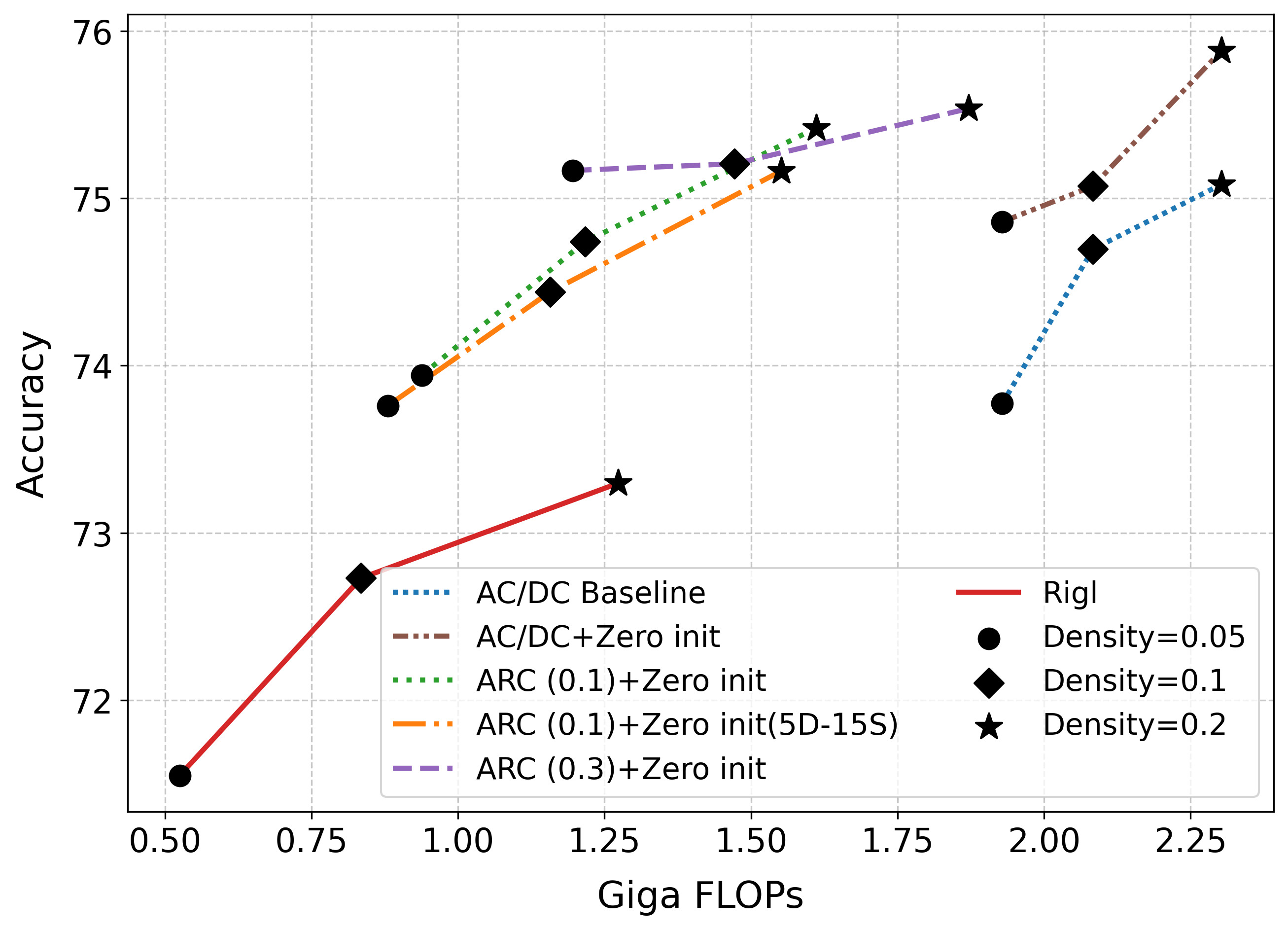}
    \caption{
    \textbf{Performance versus theoretical FLOPs.}
Accuracy--FLOP tradeoff for AC/DC, ARC variants, and RigL on
CIFAR100. The horizontal axis reports epoch-averaged theoretical
training GFLOPs per image, including support-selection probes.}
    \label{fig:cifar100_pareto}
\end{wrapfigure}

The preceding analysis motivates how to compress dense phases;
we now test whether these choices preserve their benefits over
complete training runs.
We first examine the individual modifications and their
accuracy--FLOP tradeoffs on CIFAR100, then evaluate their
transfer to ImageNet and language modeling.
We consider ResNet18 \citep{He2015DeepRL} on CIFAR100 \citep{cifar},
ResNet50 and DeiT-Small \citep{touvron2021training} on ImageNet
\citep{imagenet}, and GPT-2 124M \citep{radford2019language} on
OpenWebText \citep{Gokaslan2019OpenWeb}.
We compare against AC/DC \citep{peste2021acdc} and RigL
\citep{evci-rigl}, and include published Top-KAST
\citep{jayakumar2020topkast} results for ResNet50.
Our vision runs use three seeds; language-model results use one seed.
We report theoretical FLOPs, including support-selection probes.
Our implementation uses dense tensors with masks and does not
demonstrate wall-clock acceleration.
Training details and FLOP accounting are provided in
Appendix~\ref{appendix : training details}.

\paragraph{Shorter dense phases preserve performance.}
We first test whether dense phases must have the same duration as sparse phases.
On CIFAR100 with ResNet18, we keep the total cycle length fixed and vary the number of dense epochs per phase.
Figure~\ref{fig:pareto-dense-phases-c100} (left) shows that dense phases can be substantially shortened with little loss in accuracy.
In particular, reducing the dense portion from 10 epochs to as few as 3 epochs maintains performance within the observed variation.

This result supports the view that much of the useful rewiring can occur early in the dense phase.
Shortening dense phases, therefore, reduces theoretical FLOPs by replacing dense computation with sparse training.
The same figure also shows that zero initialization is comparable to or better than stale initialization across dense-phase lengths, suggesting that reintroduced parameters do not need to preserve stale values in order to make effective use of the rewiring window.

\paragraph{ARC matches AC/DC with sparse rewiring.}
Shortening dense phases reduces how long all coordinates are
trained; ARC addresses the complementary question of whether
all coordinates need to participate at all.
Figure~\ref{fig:pareto-dense-phases-c100} (right) compares ARC with varying $G_t$ densities on CIFAR100.
ARC with a sufficiently large rewiring density matches the performance of the AC/DC baseline while reducing theoretical FLOPs, because each rewiring phase updates only \(A_t=S_t\cup G_t\) rather than the full dense parameter set.
Increasing the density of $G_t$ improves performance: larger rewiring supports better approximate
the dense phase, while smaller supports provide greater
compression.

Table~\ref{tab:acdc_rigl_results_cifar100} reports the corresponding CIFAR100 results.
ARC(0.1) already performs comparably to the AC/DC baseline across sparsity levels, while ARC(0.3) with zero initialization achieves the best accuracy at 95\% and 90\% sparsity and is competitive with the best AC/DC variant at 80\% sparsity.
Combining all three modifications---zero initialization, shorter rewiring windows, and sparse high-gradient rewiring---also remains competitive with AC/DC, despite using substantially fewer theoretical FLOPs.

\begin{table}[ht]
    \centering
    \caption{\textbf{AC/DC, ARC and RigL on CIFAR100 with ResNet18.}}
    \label{tab:acdc_rigl_results_cifar100}
    \small
    \begin{tabular}{@{}lccc@{}}
    \toprule
     & \textbf{Sparsity=0.95} & \textbf{Sparsity=0.9} & \textbf{Sparsity=0.8} \\ 
    \midrule
    ARC (0.1) & $73.91 \pm 0.21$ & $74.45 \pm 0.14$ & $74.94 \pm 0.21$ \\
    ARC (0.1) + Zero init & $73.94 \pm 0.26$ & $74.74 \pm 0.36$ & $75.42 \pm 0.19$ \\
    ARC (0.3) + Zero init & \textbf{75.17} $\pm$ 0.98 & \textbf{75.21} $\pm$ 0.26 & \underline{75.54 $\pm$ 0.58} \\
    ARC (0.1) + Zero init 5-D, 15-S & $73.76 \pm 0.34$ & $74.44 \pm 0.06$ & $75.16 \pm 0.20$ \\
    \midrule
    AC/DC Baseline & $73.77 \pm 0.17$ & $74.70 \pm 0.40$ & $75.08 \pm 0.30$ \\
    AC/DC + Zero init & \underline{74.86} $\pm$ 0.38 & \underline{75.07 $\pm$ 0.38} & \textbf{75.88} $\pm$ 0.20 \\
    \midrule
    RigL & $71.55 \pm 0.26$ & $72.73 \pm 0.41$ & $73.30 \pm 0.14$ \\ 
    \bottomrule
    \end{tabular}
\end{table}

\begin{table}[ht]
    \centering
    \caption{
\textbf{AC/DC, ARC and RigL on ImageNet with
DeiT-Small.} We report top-1 accuracy (\%) and epoch-averaged
theoretical training GFLOPs per image, including
support-selection probes.
}
    \label{tab:acdc_rigl_results_imagenet}
    \small
    \begin{tabular}{@{}lcccc@{}}
    \toprule
     & \multicolumn{2}{c}{\textbf{Sparsity=0.7}} 
     & \multicolumn{2}{c}{\textbf{Sparsity=0.5}} \\ 
    \cmidrule(lr){2-3} \cmidrule(lr){4-5}
     & \textbf{Test Acc} & \textbf{GFLOPs}
     & \textbf{Test Acc} & \textbf{GFLOPs} \\
    \midrule
    ARC (0.2)
        & \textbf{76.50} $\pm$ \textbf{0.21} & 13.34
        & \textbf{77.82} $\pm$ \textbf{0.19} & 18.11 \\
        
    ARC (0.2) + Zero init 
        & \underline{$76.21 \pm 0.22$} & 13.34
        & $77.67 \pm 0.76$ & 18.11\\
    
    ARC (0.2) + Zero init 10-D, 20-S 
        & $76.20 \pm 0.08$ & 12.59
        & $77.40 \pm 0.27$ & 17.36 \\
        
    \midrule
    AC/DC Baseline               
        & $75.89 \pm 0.31$ & 19.01
        & $77.433 \pm 0.30$ & 21.52 \\
    
    AC/DC + Zero init                
        & 75.891 $\pm$ 0.74 & 19.01
        & \underline{77.82 $\pm$ 0.12} & 21.52 \\
        
    \midrule
    RigL                                  
        & $73.89 \pm 0.47$ & 10.32
        & $76.25 \pm 0.18$ & 15.30 \\ 
    \bottomrule
    \end{tabular}
\end{table}

\paragraph{Comparison with RigL and Pareto frontier.}
RigL is an important sparse-to-sparse baseline because it also avoids dense training phases through gradient-based rewiring.
However, Table~\ref{tab:acdc_rigl_results_cifar100} shows that all AC/DC-style variants substantially outperform RigL on CIFAR100 at the same sparsity levels.
Figure~\ref{fig:cifar100_pareto} further compares accuracy against theoretical FLOPs.
ARC variants improve the performance--FLOP frontier relative to both AC/DC and RigL: they retain AC/DC-like accuracy while shifting to lower FLOP budgets.

\paragraph{ImageNet results.}
We next evaluate ARC on ImageNet, testing architectural transfer
with DeiT-Small and high-sparsity training with ResNet50.
For DeiT-Small, Table~\ref{tab:acdc_rigl_results_imagenet}
shows that ARC(0.2) outperforms standard AC/DC and RigL
at both $50\%$ and $70\%$ sparsity, while using fewer theoretical
GFLOPs than AC/DC.
At $50\%$ sparsity, it also matches AC/DC with zero initialization.

We further evaluate ResNet50 at $95\%$ sparsity using the
100-epoch AC/DC benchmark~\citep{peste2021acdc}.
Table~\ref{tab:resnet50_imagenet_95} shows that ARC--Zero
nearly matches AC/DC accuracy with $54\%$ fewer theoretical
training FLOPs.
It also outperforms the published Top-KAST baseline in accuracy
at a moderately higher FLOP budget.
These results demonstrate that ARC's accuracy--FLOP benefits
extend across architectures and to high sparsity.
Training details are provided in
Appendix~\ref{appendix:resnet50_training}, with a
hardware-oriented 2:4 semi-structured extension in
Appendix~\ref{appendix:semi_structured}.

\begin{table}[t]
    \centering
    \caption{
\textbf{ResNet50/ImageNet at 95\% sparsity.}
Published baselines from Table~2 of \citep{peste2021acdc}
and our three-seed ARC results.
Inference FLOPs per image and full-run training FLOPs are
normalized to dense; ARC includes selection probes.
Top-KAST uses 95\% forward and 50\% backward sparsity.
}
    \label{tab:resnet50_imagenet_95}
    \small
    \setlength{\tabcolsep}{5pt}
    \begin{tabular*}{\linewidth}
        {@{\extracolsep{\fill}}lrrr@{}}
        \toprule
        Method & Top-1 (\%) & Inference FLOPs & Training FLOPs \\
        \midrule
        \multicolumn{4}{@{}l}{\textit{Published baselines}} \\
        RigL $1\times$ (ERK)
            & $69.70 \pm 0.17$ & $0.12\times$ & $0.13\times$ \\
        Top-KAST
            & $71.96$ & $0.08\times$ & $0.22\times$ \\
        AC/DC
            & $73.14 \pm 0.20$ & $0.11\times$ & $0.53\times$ \\
        \midrule
        \multicolumn{4}{@{}l}{\textit{Our runs}} \\
        ARC--Stale
            & $72.083 \pm 0.077$ & $0.0920\times$ & $0.2321\times$ \\
        ARC--Zero
            & $73.087 \pm 0.046$ & $0.1079\times$ & $0.2431\times$ \\
        \bottomrule
    \end{tabular*}
\end{table}

\paragraph{Language modeling.}
We also evaluate GPT-2 124M trained from scratch on
OpenWebText for 600,000 updates at $80\%$ sparsity.
Both ARC variants improve validation perplexity over
AC/DC and RigL (Table~\ref{tab:gpt2_openwebtext}).
ARC--Zero reduces AC/DC perplexity from $23.03$ to
$21.25$ ($7.7\%$) and improves six of seven zero-shot
task scores, increasing the average from $36.27\%$
to $37.29\%$.
Training details and task-wise results are provided in
Appendices~\ref{appendix:language_training}
and~\ref{appendix:language_results}.


\begin{table}[t]
    \centering
    \caption{
    \textbf{GPT-2 124M on OpenWebText at 80\% target sparsity.}
    Results use one training seed (1337) and the end-of-training
    checkpoint. The zero-shot average is the arithmetic mean of
    seven task scores (\%).
    Bold indicates the best result among the compared sparse methods.
    }
    \label{tab:gpt2_openwebtext}
    \small
    \setlength{\tabcolsep}{5pt}
    \begin{tabular*}{\linewidth}
        {@{\extracolsep{\fill}}lrrr@{}}
        \toprule
        Method & Validation loss $\downarrow$
               & Perplexity $\downarrow$
               & Zero-shot average $\uparrow$ \\
        \midrule
        AC/DC
            & $3.13690$ & $23.0327$ & $36.27$ \\
        RigL
            & $3.13847$ & $23.0685$ & $35.67$ \\
        \midrule
        ARC--Stale
            & $3.06056$ & $21.3395$ & $36.35$ \\
        ARC--Zero
            & $\mathbf{3.05646}$ & $\mathbf{21.2521}$
            & $\mathbf{37.29}$ \\
        \bottomrule
    \end{tabular*}
\end{table}

\section{Conclusion}
\label{sec:conclusion}

We show that the dense phases in AC/DC can be substantially compressed without losing its main benefit.
Our theory frames alternating training as sparse consolidation followed by support expansion, and our diagnostics show the corresponding empirical pattern: sparse phases refine the current support, while dense phases expose high-gradient off-support coordinates that drive mask rewiring and persistent sign correction.
This suggests that dense phases are not uniformly useful across all parameters.
Building on this insight, we introduced Zero init, truncated dense phases, and ARC, which replaces each dense phase with sparse rewiring.
Across CIFAR100 and ImageNet, ARC preserves AC/DC-level accuracy while reducing theoretical FLOPs and improving the performance--FLOP frontier relative to AC/DC and RigL.
The GPT-2 results extend these benefits to language modeling,
with lower perplexity and higher average zero-shot performance. 
This work focuses on unstructured sparsity, where theoretical FLOP reductions do not always translate directly into wall-clock speedups on current hardware. Nevertheless, unstructured sparsity remains a useful setting for studying the limits of sparse training. Its flexibility allows us to ask which parameters are needed, how sparse masks evolve, and which parts of dense training are essential. The resulting insights can inform both future sparse-training algorithms and more hardware-friendly structured variants.

\section{Acknowledgements}
We gratefully acknowledge the Gauss Centre for Supercomputing e.V. for funding this project by
providing computing time on the GCS Supercomputer JUWELS at Jülich Supercomputing Centre
(JSC). We are also grateful for funding from the European
Research Council (ERC) under the Horizon Europe Framework Programme (HORIZON) for proposal
number 101116395 SPARSE-ML.

\bibliographystyle{plain}\bibliography{example_paper}

\appendix

\newpage

\section{AC/DC and ARC Schematic}\label{appendix:acdc_arc_schematics}
\Cref{fig:acdc_schematic} shows the schematic for both AC/DC and ARC.

\begin{figure*}[ht!]
    \centering
    \includegraphics[width=\textwidth]{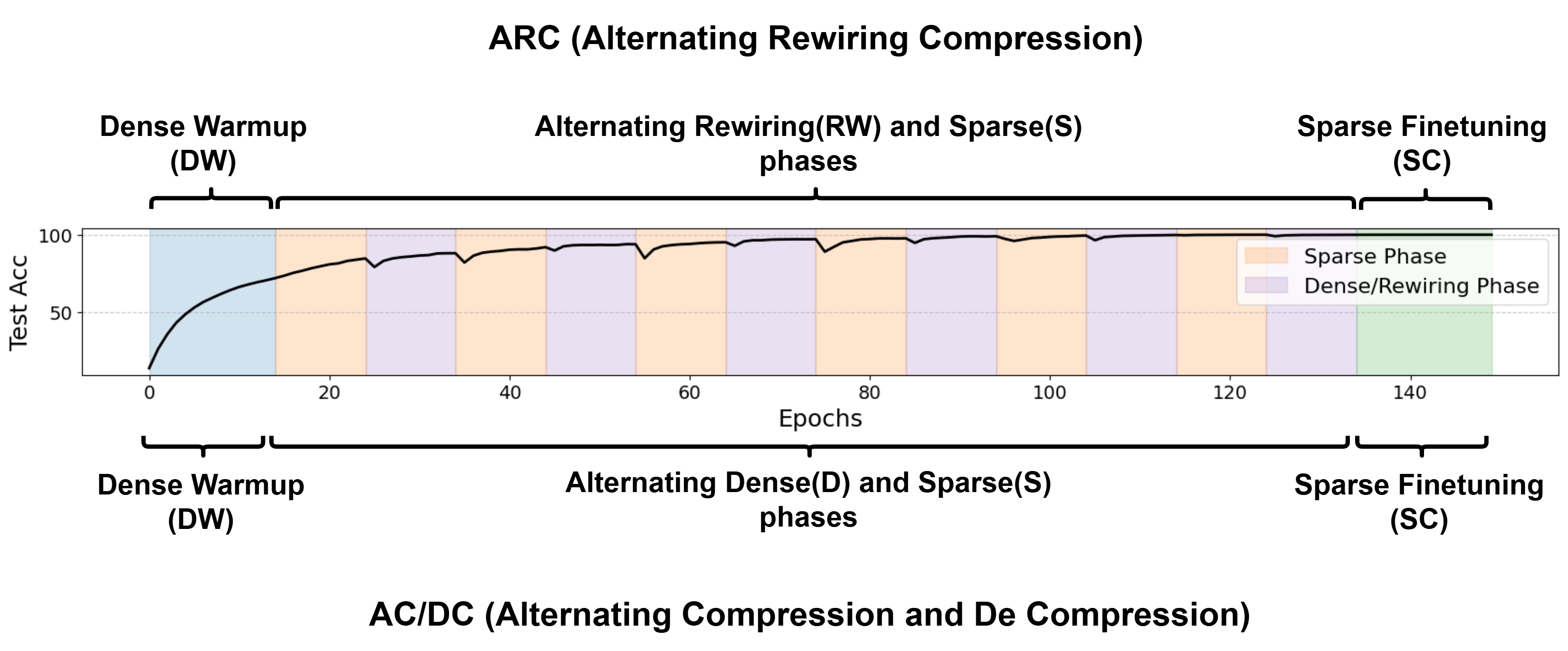}
    \caption{\textbf{AC/DC \& ARC Schematic.} AC/DC finds a sparse network, by starting with a dense warmup, then alternating between dense and sparse phases and finally a sparse finetuning phase. ARC replaces the alternating dense phases with sparse rewiring phases.}
    \label{fig:acdc_schematic}
\end{figure*}

\section{Convergence theory}\label{appendix : convergence}
The main goal of this section is to prove convergence of a modified iterative hard thresholding algorithm that replaces dense training phases with sparse ones that compress the original dense gradient.
Our theory relies on the following assumptions.
\begin{assumption}[Sparse optimum]
\label{ass:sparse-optimum}
There exists a global minimizer \(\theta^\star\in\arg\min_\theta f(\theta)\) such that
\[
    \|\theta^\star\|_0 \le k^\star.
\]
\end{assumption}

\begin{assumption}[Restricted smoothness]
\label{ass:restricted-smoothness}
The objective \(f:\mathbb R^N\to\mathbb R\) is \((s,\beta)\)-smooth with
\[
    s = 2k + q + 3k^\star.
\]
That is, for all \(\theta\) and all \(s\)-sparse perturbations \(\delta\),
\[
    f(\theta+\delta)
    \le
    f(\theta)
    +
    \langle \nabla f(\theta),\delta\rangle
    +
    \frac{\beta}{2}\|\delta\|_2^2.
\]
\end{assumption}

\begin{assumption}[Concentrated PL condition]
\label{ass:cpl}
There exists $\alpha>0$ such that $f$ satisfies the
$(k^\star,\alpha)$-concentrated PL condition defined in
\eqref{eq:concentrated_pl_def}.
\end{assumption}

\begin{assumption}[Compressed-gradient accuracy]
\label{ass:compressed-gradient-accuracy}
Let
\[
    \Delta_t = f(\theta_t)-f(\theta^\star).
\]
For every support \(U_t\) arising in the IHT proof, with
\[
    |U_t|\le 2k+q+3k^\star,
\]
the compressed gradient \(h_t=P_{A_t}g_t\) satisfies
\[
    \mathbb E\!\left[
    \|(h_t-\nabla f(\theta_t))_{U_t}\|_2^2
    \,\middle|\, \theta_t
    \right]
    \le
    \sigma_{\rm cg}^2 + \rho\,\Delta_t.
\]
Here \(\sigma_{\rm cg}^2\) is the stochastic-gradient variance remaining after projection,
and \(\rho\Delta_t\) is the deterministic error introduced by gradient compression.
\end{assumption}

\begin{lemma}[One-step AC/DC bound with compressed gradients]
\label{lem:one-step-compressed-iht}
Suppose Assumptions~\ref{ass:sparse-optimum},
\ref{ass:restricted-smoothness}, and~\ref{ass:cpl} hold. Let
\[
    \theta_{t+1}
    =
    T_k(\theta_t-\eta h_t),
    \qquad
    \eta=\frac{1}{2\beta}.
\]
There exist universal constants \(C_0,C_1>0\) such that, whenever
\[
    k \ge C_0 k^\star\left(\frac{\beta}{\alpha}\right)^2,
\]
we have
\[
    \mathbb E[\Delta_{t+1}\mid \theta_t]
    \le
    \left(
    1-\frac{\alpha}{32\beta}
    \right)\Delta_t
    +
    C_1\beta\,
    \mathbb E\!\left[
    \|(h_t-\nabla f(\theta_t))_{U_t}\|_2^2
    \,\middle|\,\theta_t
    \right].
\]
\end{lemma}

\begin{proof}
Let
\[
    S=\supp(\theta_t),\qquad
    S_0=\supp(\theta_{t+1}),\qquad
    S^\star=\supp(\theta^\star),
\]
and let \(U_t\) denote the union of the supports appearing in the IHT projection
argument:
\[
    U_t \supseteq S\cup S_0\cup S^\star\cup A_t,
    \qquad
    |U_t|\le 2k+q+3k^\star.
\]
By restricted smoothness,
\[
    f(\theta_{t+1})
    \le
    f(\theta_t)
    +
    \langle \nabla f(\theta_t), \theta_{t+1}-\theta_t\rangle
    +
    \frac{\beta}{2}\|\theta_{t+1}-\theta_t\|_2^2.
\]
As in the stochastic-IHT proof for AC/DC, insert and subtract the ideal dense-gradient
step:
\[
    \theta_t-\frac{1}{\beta}\nabla f(\theta_t).
\]
This gives
\[
\begin{aligned}
    f(\theta_{t+1})
    &\le
    f(\theta_t)
    +
    \frac{\beta}{2}
    \left\|
        \theta_{t+1}
        -
        \left(
            \theta_t-\frac{1}{\beta}\nabla f(\theta_t)
        \right)
    \right\|_2^2
    -
    \frac{1}{2\beta}\|\nabla f(\theta_t)\|_{U_t}^2 .
\end{aligned}
\]
Now write
\[
    \theta_{t+1}
    =
    T_k(\theta_t-\eta h_t).
\]
The standard AC/DC/IHT projection argument controls the error introduced by applying
\(T_k\), provided
\[
    k \gtrsim k^\star(\beta/\alpha)^2.
\]
In the exact-gradient case \(h_t=\nabla f(\theta_t)\), this yields
\[
    \Delta_{t+1}
    \le
    \left(1-\frac{\alpha}{32\beta}\right)\Delta_t .
\]
When \(h_t\neq \nabla f(\theta_t)\), every occurrence of the stochastic gradient in the
AC/DC proof is replaced by
\[
    h_t
    =
    \nabla f(\theta_t)
    +
    e_t,
    \qquad
    e_t=h_t-\nabla f(\theta_t).
\]
The additional terms are supported on \(U_t\) and are bounded by Young's inequality:
\[
    2|\langle a,e_t\rangle|
    \le
    \lambda \|a\|_2^2+\lambda^{-1}\|e_t\|_2^2.
\]
Choosing \(\lambda\) as in the standard IHT proof, the terms involving
\(\|a\|_2^2\) are absorbed into the dense-gradient contraction, while the remaining
contribution is at most
\[
    C_1\beta \|(h_t-\nabla f(\theta_t))_{U_t}\|_2^2.
\]
Taking conditional expectation gives the claim.
\end{proof}

\begin{theorem}[AC/DC with compressed gradient]
\label{thm:compressed-growth-acdc}
Suppose Assumptions~\ref{ass:sparse-optimum}--\ref{ass:compressed-gradient-accuracy}
hold and
\[
    k \ge C_0 k^\star\left(\frac{\beta}{\alpha}\right)^2.
\]
Define
\[
    \lambda_{\rm cg}
    =
    \frac{\alpha}{32\beta}
    -
    C_1\beta\rho .
\]
If
\[
    \lambda_{\rm cg}>0,
\]
then the compressed-gradient AC/DC update satisfies
\[
    \mathbb E[\Delta_{t+1}]
    \le
    (1-\lambda_{\rm cg})\mathbb E[\Delta_t]
    +
    C_1\beta\sigma_{\rm cg}^2.
\]
Consequently, after \(T\) iterations,
\[
    \mathbb E[\Delta_T]
    \le
    (1-\lambda_{\rm cg})^T\Delta_0
    +
    \frac{C_1\beta\sigma_{\rm cg}^2}{\lambda_{\rm cg}}.
\]
Equivalently, for any tolerance \(\varepsilon>0\), after
\[
    T
    =
    O\!\left(
        \frac{1}{\lambda_{\rm cg}}
        \log\frac{\Delta_0}{\varepsilon}
    \right)
\]
iterations, the method reaches
\[
    \mathbb E[f(\theta_T)-f(\theta^\star)]
    \le
    \varepsilon
    +
    O\!\left(
        \frac{\beta\sigma_{\rm cg}^2}{\lambda_{\rm cg}}
    \right).
\]
\end{theorem}

\begin{proof}
By Lemma~\ref{lem:one-step-compressed-iht},
\[
    \mathbb E[\Delta_{t+1}\mid \theta_t]
    \le
    \left(
    1-\frac{\alpha}{32\beta}
    \right)\Delta_t
    +
    C_1\beta\,
    \mathbb E\!\left[
    \|(h_t-\nabla f(\theta_t))_{U_t}\|_2^2
    \,\middle|\,\theta_t
    \right].
\]
Using Assumption~\ref{ass:compressed-gradient-accuracy},
\[
    \mathbb E[\Delta_{t+1}\mid \theta_t]
    \le
    \left(
    1-\frac{\alpha}{32\beta}
    +
    C_1\beta\rho
    \right)\Delta_t
    +
    C_1\beta\sigma_{\rm cg}^2.
\]
By definition,
\[
    \lambda_{\rm cg}
    =
    \frac{\alpha}{32\beta}
    -
    C_1\beta\rho,
\]
so
\[
    \mathbb E[\Delta_{t+1}\mid \theta_t]
    \le
    (1-\lambda_{\rm cg})\Delta_t
    +
    C_1\beta\sigma_{\rm cg}^2.
\]
Taking total expectation and unrolling the recursion gives
\[
    \mathbb E[\Delta_T]
    \le
    (1-\lambda_{\rm cg})^T\Delta_0
    +
    C_1\beta\sigma_{\rm cg}^2
    \sum_{j=0}^{T-1}(1-\lambda_{\rm cg})^j.
\]
Since
\[
    \sum_{j=0}^{T-1}(1-\lambda_{\rm cg})^j
    \le
    \frac{1}{\lambda_{\rm cg}},
\]
we obtain
\[
    \mathbb E[\Delta_T]
    \le
    (1-\lambda_{\rm cg})^T\Delta_0
    +
    \frac{C_1\beta\sigma_{\rm cg}^2}{\lambda_{\rm cg}}.
\]
The stated iteration complexity follows by choosing \(T\) so that
\[
    (1-\lambda_{\rm cg})^T\Delta_0\le \varepsilon.
\]
\end{proof}

\subsection{Verifying the compressed-gradient assumption}

The previous theorem reduces the analysis to proving that
\[
    P_{A_t}g_t
\]
is a sufficiently accurate proxy for the dense gradient on the supports relevant to the IHT proof. The next proposition shows that this follows from gradient compressibility outside the inherited sparse mask.

\begin{assumption}[Old support plus growth support captures gradient energy]
\label{ass:support-capture}
For every iterate \(\theta_t\), let \(S_t=\supp(\theta_t)\). Suppose that there exists
a set \(B_t\subseteq S_t^c\), with \(|B_t|=q\), such that
\[
    \|(I-P_{S_t\cup B_t})\nabla f(\theta_t)\|_{U_t}^2
    \le
    \tau + \rho \Delta_t .
\]
\end{assumption}

\begin{proposition}[Top-gradient growth gives the best proxy support]
\label{prop:top-gradient-support}
Let
\[
    G_t
    =
    \Top_q\!\left((\nabla f(\theta_t))_{S_t^c}\right),
    \qquad
    A_t=S_t\cup G_t.
\]
If Assumption~\ref{ass:support-capture} holds, then
\[
    \|(I-P_{A_t})\nabla f(\theta_t)\|_{U_t}^2
    \le
    \tau+\rho\Delta_t.
\]
Therefore, top-gradient growth captures at least as much dense-gradient energy as any other \(q\)-coordinate growth set outside the inherited support.
\end{proposition}

\begin{proof}
The operator \(\Top_q\) gives the best \(q\)-sparse approximation in Euclidean norm.
Therefore, among all sets \(B\subseteq S_t^c\) with \(|B|=q\),
\[
    \|(I-P_{S_t\cup G_t})\nabla f(\theta_t)\|_{U_t}^2
    \le
    \|(I-P_{S_t\cup B})\nabla f(\theta_t)\|_{U_t}^2.
\]
Taking \(B=B_t\) from Assumption~\ref{ass:support-capture} proves the result.
\end{proof}

\begin{proposition}[From support capture to compressed-gradient accuracy]
\label{prop:support-capture-to-oracle}
Suppose \(g_t\) is a stochastic gradient satisfying
\[
    \mathbb E[g_t\mid \theta_t]=\nabla f(\theta_t),
    \qquad
    \mathbb E[\|(g_t-\nabla f(\theta_t))_{A_t}\|_2^2\mid \theta_t]
    \le
    \sigma_A^2.
\]
Assume also that
\[
    \|(I-P_{A_t})\nabla f(\theta_t)\|_{U_t}^2
    \le
    \tau+\rho\Delta_t.
\]
Then \(h_t=P_{A_t}g_t\) satisfies Assumption~\ref{ass:compressed-gradient-accuracy}
with
\[
    \sigma_{\rm cg}^2 = 2\sigma_A^2+2\tau,
    \qquad
    \rho_{\rm cg}=2\rho.
\]
\end{proposition}

\begin{proof}
We decompose
\[
    h_t-\nabla f(\theta_t)
    =
    P_{A_t}(g_t-\nabla f(\theta_t))
    -
    (I-P_{A_t})\nabla f(\theta_t).
\]
Restricting to \(U_t\) and using \(\|a+b\|_2^2\le 2\|a\|_2^2+2\|b\|_2^2\),
\[
\begin{aligned}
    \|(h_t-\nabla f(\theta_t))_{U_t}\|_2^2
    &\le
    2\|P_{A_t}(g_t-\nabla f(\theta_t))\|_2^2
    +
    2\|(I-P_{A_t})\nabla f(\theta_t)\|_{U_t}^2 .
\end{aligned}
\]
Taking conditional expectation gives
\[
    \mathbb E[
    \|(h_t-\nabla f(\theta_t))_{U_t}\|_2^2
    \mid \theta_t]
    \le
    2\sigma_A^2
    +
    2\tau
    +
    2\rho\Delta_t.
\]
Thus Assumption~\ref{ass:compressed-gradient-accuracy} holds with
\[
    \sigma_{\rm cg}^2=2\sigma_A^2+2\tau,
    \qquad
    \rho_{\rm cg}=2\rho.
\]
\end{proof}

\section{Saddle-to-saddle learning dynamics explain benefits of alternating phases}
\label{appendix : saddle dynamics}
The convergence theory leaves us to understand the role of the alternating sparse phases during training. 
Our main intuitive argument is that saddle-to-saddle dynamics that learn singular directions in order of the dominant singular values focus learning on a low dimensional subspace.
As long as the dynamics can be compressed to a sparse support, we naturally have to change the support when we enter a new learning phase to expand our low dimensional subspace by a new singular direction.

Again, we consider a three-layer linear network with one hidden layer,
\[
    f_{W_1,W_2}(x)=W_2W_1x,
\]
trained on a population squared loss with whitened inputs. Let
\[
    M = \mathbb E[yx^\top]
\]
denote the target input-output map, and write its singular value decomposition as
\[
    M = U \operatorname{diag}(s_1,\ldots,s_r)V^\top,
    \qquad s_1>s_2>\cdots>s_r>0.
\]
The population loss is
\[
    \mathcal L(W_1,W_2)
    =
    \frac12\|W_2W_1-M\|_F^2.
\]
We write
\[
    \theta = (\operatorname{vec}(W_1),\operatorname{vec}(W_2)),
    \qquad
    g(\theta)=\nabla \mathcal L(\theta).
\]
Similarly to previous section, for a coordinate set \(A\), let \(P_A\) denote the coordinate projection that keeps the
coordinates in \(A\) and sets all other coordinates to zero.
At the start of a sparse-growth phase, let \(S\) be the support inherited from the
previous magnitude-pruned phase. The growth phase activates
\[
    A = S \cup G,
\]
where
\[
    G
    =
    \operatorname{Top}_q\!\left(g(\theta_0)_{S^c}\right)
\]
is the set of the \(q\) largest-magnitude gradient coordinates outside \(S\), computed at
the beginning of the phase. During the phase we train only coordinates in \(A\):
\[
    \dot \theta_A(t)
    =
    -P_A g(\theta_A(t)).
\]
The dense gradient flow is
\[
    \dot \theta_\star(t)
    =
    -g(\theta_\star(t)).
\]

\begin{assumption}[Gradient capture by the inherited-plus-growth support in rewiring phase]
\label{ass:gradient-capture}
Along the rewiring phase, the active support \(A=S\cup G\) captures all but an \(\varepsilon\)-fraction of the dense gradient energy:
\[
    \|(I-P_A)g(\theta_A(t))\|_2
    \le
    \varepsilon \|g(\theta_A(t))\|_2
    \qquad
    \text{for all } t\in[0,T].
\]
\end{assumption}

\begin{theorem}[Sparse-growth phases are dense-gradient perturbations]
\label{thm:sparse-growth-perturbation}
Assume \(\mathcal L\) is \(L\)-smooth on a tube containing the dense and sparse-growth
trajectories. Suppose Assumption~\ref{ass:gradient-capture} holds and that
\(\|g(\theta_A(t))\|_2\le G\) for \(t\in[0,T]\). If
\(\theta_A(0)=\theta_\star(0)\), then
\[
    \|\theta_A(t)-\theta_\star(t)\|_2
    \le
    \frac{\varepsilon G}{L}\left(e^{Lt}-1\right)
    \qquad
    \text{for all }t\in[0,T].
\]
Moreover, the instantaneous loss decrease during the sparse-growth phase satisfies
\[
    \frac{d}{dt}\mathcal L(\theta_A(t))
    =
    -\|P_A g(\theta_A(t))\|_2^2
    \le
    -(1-\varepsilon^2)\|g(\theta_A(t))\|_2^2.
\]
Thus the sparse-growth phase achieves at least a \(1-\varepsilon^2\) fraction of the
dense gradient-flow decrease in loss.
\end{theorem}

\begin{proof}
Write the sparse-growth dynamics as
\[
    \dot \theta_A
    =
    -P_Ag(\theta_A)
    =
    -g(\theta_A) + r_A(\theta_A),
\]
where
\[
    r_A(\theta_A)=(I-P_A)g(\theta_A).
\]
By Assumption~\ref{ass:gradient-capture},
\[
    \|r_A(\theta_A(t))\|_2
    \le
    \varepsilon \|g(\theta_A(t))\|_2
    \le
    \varepsilon G.
\]
Let \(e(t)=\theta_A(t)-\theta_\star(t)\). Then
\[
    \dot e(t)
    =
    -g(\theta_A(t))+g(\theta_\star(t))+r_A(\theta_A(t)).
\]
Using \(L\)-Lipschitzness of \(g\),
\[
    \frac{d}{dt}\|e(t)\|_2
    \le
    L\|e(t)\|_2+\varepsilon G.
\]
Since \(e(0)=0\), Gronwall's inequality gives
\[
    \|e(t)\|_2
    \le
    \frac{\varepsilon G}{L}\left(e^{Lt}-1\right).
\]

For the loss decrease,
\[
    \frac{d}{dt}\mathcal L(\theta_A(t))
    =
    \langle g(\theta_A(t)),\dot\theta_A(t)\rangle
    =
    -\langle g(\theta_A(t)),P_Ag(\theta_A(t))\rangle.
\]
Since \(P_A\) is an orthogonal coordinate projection,
\[
    \langle g,P_Ag\rangle = \|P_Ag\|_2^2.
\]
Furthermore,
\[
    \|P_Ag\|_2^2
    =
    \|g\|_2^2-\|(I-P_A)g\|_2^2
    \ge
    (1-\varepsilon^2)\|g\|_2^2,
\]
which proves the claim.
\end{proof}

\subsection{Preservation of SVD-mode learning order}

Following the dense deep-linear analysis of Saxe et al., suppose the dense trajectory
is initialized on the SVD-aligned manifold
\[
    W_2 = U\operatorname{diag}(b_i)R^\top,
    \qquad
    W_1 = R\operatorname{diag}(a_i)V^\top.
\]
Then the dense dynamics decouple into singular modes:
\[
    \dot a_i = b_i(s_i-a_ib_i),
    \qquad
    \dot b_i = a_i(s_i-a_ib_i).
\]
Let
\[
    m_i(\theta)=a_i(\theta)b_i(\theta)
\]
denote the learned strength of singular mode \(i\).

\begin{assumption}[Mode-separation margin]
\label{ass:mode-gap}
On a time interval \(I\subseteq[0,T]\), the dense dynamics have a strict mode velocity
gap:
\[
    \dot m_i^\star(t)-\dot m_j^\star(t)\ge \gamma
    \qquad
    \text{for all }t\in I,
\]
for some \(\gamma>0\), where \(i\) is the dominant mode and \(j\) is a competing lower
singular mode.
\end{assumption}

\begin{theorem}[Sparse growth preserves mode ordering]
\label{thm:mode-order}
Assume Assumptions~\ref{ass:gradient-capture} and~\ref{ass:mode-gap}. Suppose the mode
coordinates \(m_i\) and \(m_j\) are differentiable in a neighborhood of the trajectories
and satisfy
\[
    \|\nabla m_i(\theta)\|_2\le B,
    \qquad
    \|\nabla m_j(\theta)\|_2\le B.
\]
If
\[
    2B\varepsilon G < \gamma,
\]
then the sparse-growth phase preserves the dense ordering of the two mode velocities:
\[
    \dot m_i^A(t) > \dot m_j^A(t)
    \qquad
    \text{for all }t\in I.
\]
In particular, whenever the dense dynamics learn mode \(i\) before mode \(j\) with
a margin \(\gamma\), the sparse-growth phase preserves this ordering up to perturbation error controlled by the discarded gradient tail.
\end{theorem}

\begin{proof}
Along dense flow,
\[
    \dot m_\ell^\star(t)
    =
    -\langle \nabla m_\ell(\theta_\star(t)), g(\theta_\star(t))\rangle.
\]
Along sparse-growth flow,
\[
    \dot m_\ell^A(t)
    =
    -\langle \nabla m_\ell(\theta_A(t)), P_Ag(\theta_A(t))\rangle.
\]
Ignoring the trajectory-displacement error, which is controlled by
Theorem~\ref{thm:sparse-growth-perturbation}, the direct vector-field perturbation is
\[
    \left|
    \langle \nabla m_\ell(\theta),g(\theta)\rangle
    -
    \langle \nabla m_\ell(\theta),P_Ag(\theta)\rangle
    \right|
    =
    \left|
    \langle \nabla m_\ell(\theta),(I-P_A)g(\theta)\rangle
    \right|.
\]
By Cauchy--Schwarz and Assumption~\ref{ass:gradient-capture},
\[
    \left|
    \langle \nabla m_\ell(\theta),(I-P_A)g(\theta)\rangle
    \right|
    \le
    B\varepsilon G.
\]
Therefore each mode velocity changes by at most \(B\varepsilon G\). Hence the pairwise
gap changes by at most \(2B\varepsilon G\). Since the dense gap is at least \(\gamma\),
we have
\[
    \dot m_i^A(t)-\dot m_j^A(t)
    \ge
    \gamma - 2B\varepsilon G
    >
    0.
\]
Thus the sparse-growth phase preserves the dense ordering.
\end{proof}

\subsection{Why sparse initialization can induce gradient compressibility}

The previous results reduce the problem to verifying
Assumption~\ref{ass:gradient-capture}. We now give a sufficient condition under which
the dense gradient is entrywise compressible.

Let
\[
    E = M-W_2W_1
\]
denote the residual target map. The gradients are
\[
    \nabla_{W_2}\mathcal L
    =
    -EW_1^\top,
    \qquad
    \nabla_{W_1}\mathcal L
    =
    -W_2^\top E.
\]
Suppose the residual is approximately low rank in the target singular basis:
\[
    E
    =
    \sum_{\ell=1}^r \rho_\ell u_\ell v_\ell^\top + \Delta.
\]
Define the hidden responses
\[
    h_\ell = W_1v_\ell,
    \qquad
    z_\ell = W_2^\top u_\ell.
\]
Then
\[
    \nabla_{W_2}\mathcal L
    =
    -\sum_{\ell=1}^r \rho_\ell u_\ell h_\ell^\top
    -
    \Delta W_1^\top,
\]
and
\[
    \nabla_{W_1}\mathcal L
    =
    -\sum_{\ell=1}^r \rho_\ell z_\ell v_\ell^\top
    -
    W_2^\top \Delta.
\]
Thus the gradients are sums of outer products involving the hidden responses
\(h_\ell\) and \(z_\ell\).

\begin{lemma}[Outer-product compressibility]
\label{lem:outer-product-compressibility}
Let \(x\in\mathbb R^m\) and \(y\in\mathbb R^n\). Suppose there exist coordinate sets
\(I\subseteq[m]\) and \(J\subseteq[n]\) such that
\[
    \|x_{I^c}\|_2\le \eta_x\|x\|_2,
    \qquad
    \|y_{J^c}\|_2\le \eta_y\|y\|_2.
\]
Then the support \(I\times J\) captures at least
\[
    (1-\eta_x^2)(1-\eta_y^2)
\]
of the Frobenius energy of \(xy^\top\). Equivalently,
\[
    \|(I-P_{I\times J})(xy^\top)\|_F
    \le
    \sqrt{\eta_x^2+\eta_y^2-\eta_x^2\eta_y^2}
    \,
    \|xy^\top\|_F.
\]
Since entrywise top-\(k\) gives the best \(k\)-sparse approximation in Frobenius norm,
the same upper bound holds for \(T_k(xy^\top)\) with \(k=|I||J|\).
\end{lemma}

\begin{proof}
We have
\[
    \|P_{I\times J}(xy^\top)\|_F^2
    =
    \|x_I\|_2^2\|y_J\|_2^2.
\]
By assumption,
\[
    \|x_I\|_2^2
    \ge
    (1-\eta_x^2)\|x\|_2^2,
    \qquad
    \|y_J\|_2^2
    \ge
    (1-\eta_y^2)\|y\|_2^2.
\]
Therefore
\[
    \|P_{I\times J}(xy^\top)\|_F^2
    \ge
    (1-\eta_x^2)(1-\eta_y^2)\|x\|_2^2\|y\|_2^2.
\]
Since
\[
    \|xy^\top\|_F^2=\|x\|_2^2\|y\|_2^2,
\]
the residual satisfies
\[
    \|(I-P_{I\times J})(xy^\top)\|_F^2
    \le
    \left(
    1-(1-\eta_x^2)(1-\eta_y^2)
    \right)
    \|xy^\top\|_F^2.
\]
This gives the stated bound. The final claim follows because entrywise top-\(k\) is the
minimizer of the Frobenius error over all \(k\)-sparse matrices.
\end{proof}

\begin{assumption}[Sparse-response condition induced by sparse initialization]
\label{ass:sparse-response}
For the residual modes that dominate training during the current phase, the hidden responses \(h_\ell=W_1v_\ell\) and \(z_\ell=W_2^\top u_\ell\) are entrywise compressible. That is, for each \(\ell\), there exist supports
\(H_\ell,Z_\ell\) with
\[
    |H_\ell|=p_{h,\ell},
    \qquad
    |Z_\ell|=p_{z,\ell},
\]
such that
\[
    \|(h_\ell)_{H_\ell^c}\|_2\le \eta_{h,\ell}\|h_\ell\|_2,
    \qquad
    \|(z_\ell)_{Z_\ell^c}\|_2\le \eta_{z,\ell}\|z_\ell\|_2.
\]
\end{assumption}

\begin{proposition}[Sparse responses imply compressible dense gradients]
\label{prop:gradient-compressibility}
Assume the residual decomposition above and Assumption~\ref{ass:sparse-response}.
Let
\[
    K_2 = d_{\rm out}\sum_{\ell=1}^r p_{h,\ell},
    \qquad
    K_1 = d_{\rm in}\sum_{\ell=1}^r p_{z,\ell}.
\]
Then
\[
    \left\|
    \nabla_{W_2}\mathcal L
    -
    T_{K_2}\!\left(\nabla_{W_2}\mathcal L\right)
    \right\|_F
    \le
    \sum_{\ell=1}^r
    |\rho_\ell|\,\eta_{h,\ell}\|h_\ell\|_2
    +
    \|\Delta\|_F\|W_1\|_{\rm op},
\]
and
\[
    \left\|
    \nabla_{W_1}\mathcal L
    -
    T_{K_1}\!\left(\nabla_{W_1}\mathcal L\right)
    \right\|_F
    \le
    \sum_{\ell=1}^r
    |\rho_\ell|\,\eta_{z,\ell}\|z_\ell\|_2
    +
    \|W_2\|_{\rm op}\|\Delta\|_F.
\]
Consequently, if the residual is dominated by a small number of modes and the sparse
initialization induces compressible hidden responses, then the dense gradient is
entrywise compressible.
\end{proposition}

\begin{proof}
Consider the \(W_2\) gradient:
\[
    \nabla_{W_2}\mathcal L
    =
    -\sum_{\ell=1}^r \rho_\ell u_\ell h_\ell^\top
    -
    \Delta W_1^\top.
\]
For each term \(u_\ell h_\ell^\top\), take the full support of \(u_\ell\) and the
compressible support \(H_\ell\) of \(h_\ell\). Applying
Lemma~\ref{lem:outer-product-compressibility} with \(\eta_x=0\) and
\(\eta_y=\eta_{h,\ell}\) gives
\[
    \left\|
    u_\ell h_\ell^\top
    -
    P_{[d_{\rm out}]\times H_\ell}(u_\ell h_\ell^\top)
    \right\|_F
    \le
    \eta_{h,\ell}\|u_\ell h_\ell^\top\|_F.
\]
Since \(\|u_\ell\|_2=1\), this becomes
\[
    \left\|
    u_\ell h_\ell^\top
    -
    P_{[d_{\rm out}]\times H_\ell}(u_\ell h_\ell^\top)
    \right\|_F
    \le
    \eta_{h,\ell}\|h_\ell\|_2.
\]
Taking the union of these supports over \(\ell\) gives a support of size at most \(K_2\).
By the triangle inequality,
\[
    \left\|
    \nabla_{W_2}\mathcal L
    -
    P_{\Gamma_2}\nabla_{W_2}\mathcal L
    \right\|_F
    \le
    \sum_{\ell=1}^r
    |\rho_\ell|\eta_{h,\ell}\|h_\ell\|_2
    +
    \|\Delta W_1^\top\|_F.
\]
Finally,
\[
    \|\Delta W_1^\top\|_F
    \le
    \|\Delta\|_F\|W_1\|_{\rm op}.
\]
Since \(T_{K_2}\) is the best \(K_2\)-sparse approximation, it does no worse than the
particular support \(\Gamma_2\), proving the first bound.

The proof for \(\nabla_{W_1}\mathcal L\) is identical, using
\[
    \nabla_{W_1}\mathcal L
    =
    -\sum_{\ell=1}^r \rho_\ell z_\ell v_\ell^\top
    -
    W_2^\top \Delta,
\]
the full support of \(v_\ell\), and the compressible support \(Z_\ell\) of \(z_\ell\).
\end{proof}

\begin{corollary}[Sufficient condition for rewiring phase]
\label{cor:saxe-under-sparse-growth}
Suppose the inherited support \(S\) contains the large-magnitude coordinates produced
by previously learned modes, and the added support \(G\) contains the largest
coordinates of the gradient outside \(S\). If Proposition~\ref{prop:gradient-compressibility}
implies
\[
    \|(I-P_{S\cup G})g(\theta)\|_2
    \le
    \varepsilon\|g(\theta)\|_2
\]
with \(\varepsilon\) small enough that
\[
    2B\varepsilon G < \gamma,
\]
then the rewiring phase preserves the dominant-singular-mode ordering of the dense saddle-to-saddle dynamics on the interval where the dense mode gap is at least \(\gamma\).
\end{corollary}

\textbf{Summary.} The main intuitive message is: Sparse growth preserves the saddle-to-saddle dynamics when $S \cup G$ captures most dense-gradient energy.

Sparse initialization alone does not guarantee entrywise
gradient compressibility for arbitrary singular vectors.
The sparse-response condition provides a sufficient condition
for such compressibility, consistent with the motivation
behind our empirical diagnostics.

\subsection{Alternating sparse phases as sparse saddle-to-saddle dynamics}

We now interpret alternating sparse phases through the lens of saddle-to-saddle
learning dynamics. Dense gradient flow in deep linear networks learns the target map
by successively increasing the effective rank of the end-to-end predictor. We argue
that alternating sparse phases implement the same mechanism in a sparse parameter
space: one phase consolidates the current low-rank sparse solution, while the next
phase rewires the support to activate coordinates needed for the next singular mode.

Let
\[
    \mathcal L(W_1,W_2)
    =
    \frac12 \|W_2W_1-M\|_F^2,
    \qquad
    M = U \operatorname{diag}(s_1,\ldots,s_r)V^\top,
\]
where
\[
    s_1>s_2>\cdots>s_r>0.
\]
Let \(\theta=(W_1,W_2)\), and let \(g(\theta)=\nabla \mathcal L(\theta)\).
For a support \(S\), write \(P_S\) for coordinate projection onto \(S\), and define the
support-restricted flow
\[
    \dot \theta
    =
    -P_S g(\theta).
\]

We consider an alternating sparse algorithm indexed by phases \(m=0,1,\ldots,r-1\).
At phase \(m\), the algorithm has a sparse support \(S_m\). It first trains on \(S_m\),
then constructs a growth support
\[
    G_m
    =
    \operatorname{Top}_q\!\left(g(\theta_m)_{S_m^c}\right),
\]
and trains on the enlarged support
\[
    A_m = S_m \cup G_m.
\]
After the growth phase, the support is pruned by weight magnitude:
\[
    S_{m+1}
    =
    \operatorname{Top}_s(\theta_{m+1}).
\]

\paragraph{Low-dimensional spectral phases.}
Let \(\mathcal M_m\) denote the manifold of networks whose end-to-end map captures the
first \(m\) singular modes,
\[
    W_2W_1
    \approx
    M_m
    :=
    \sum_{i=1}^m s_i u_i v_i^\top.
\]
In dense saddle-to-saddle dynamics, the trajectory moves through neighborhoods of
\(\mathcal M_0,\mathcal M_1,\ldots,\mathcal M_r\). Each \(\mathcal M_m\) behaves as a
saddle for the full problem: it is stable along already learned modes and unstable in
directions that activate mode \(m+1\).

\begin{assumption}[Sparse representation of the current saddle]
\label{ass:sparse-current-mode}
For each \(m\), there exists a support \(S_m^\star\) of size at most \(s\) and a point
\(\theta_m^\star\) supported on \(S_m^\star\) such that
\[
    \|W_2^\star W_1^\star - M_m\|_F \le \delta_m.
\]
Moreover, the support \(S_m\) produced by magnitude pruning satisfies
\[
    \|\theta_m - P_{S_m}\theta_m\|_2 \le \delta_m,
\]
and \(S_m\) contains a support sufficient to represent \(M_m\) up to error \(\delta_m\).
\end{assumption}

\begin{assumption}[Residual dominance by the next singular mode]
\label{ass:residual-dominance}
Near \(\mathcal M_m\), the residual is dominated by the next singular component:
\[
    E_m(\theta)
    :=
    M-W_2W_1
    =
    s_{m+1}u_{m+1}v_{m+1}^\top + R_m(\theta),
\]
with
\[
    \|R_m(\theta)\|_F
    \le
    \eta_m s_{m+1},
    \qquad
    \eta_m<1.
\]
\end{assumption}

\begin{assumption}[Sparse escape direction]
\label{ass:sparse-escape}
Let \(e_{m+1}(\theta)\) denote the dense escape component of the gradient that increases
the learned strength of singular mode \(m+1\). The enlarged support
\[
    A_m = S_m\cup G_m
\]
captures this escape direction up to error \(\varepsilon_m\):
\[
    \|(I-P_{A_m})e_{m+1}(\theta)\|_2
    \le
    \varepsilon_m \|e_{m+1}(\theta)\|_2.
\]
\end{assumption}

\begin{assumption}[No escape on the old support]
\label{ass:old-support-stall}
The old support \(S_m\) has negligible projection onto the next escape direction:
\[
    \|P_{S_m}e_{m+1}(\theta)\|_2
    \le
    \alpha_m \|e_{m+1}(\theta)\|_2,
\]
where \(\alpha_m\ll 1\).
\end{assumption}

Assumption~\ref{ass:old-support-stall} formalizes the need for rewiring: training only
on \(S_m\) can consolidate the current mode, but it cannot efficiently move toward the
next singular direction. Assumption~\ref{ass:sparse-escape} states that the
top-gradient growth set \(G_m\) supplies the missing coordinates needed to escape.

\begin{theorem}[Alternating sparse phases implement sparse saddle-to-saddle dynamics]
\label{thm:alternating-sparse-saddle}
Suppose Assumptions~\ref{ass:sparse-current-mode}--\ref{ass:old-support-stall} hold
for phases \(m=0,\ldots,r-1\). Assume also that \(\mathcal L\) is \(L\)-smooth in a
neighborhood of the relevant trajectories.

Then each alternating cycle has the following behavior.

\begin{enumerate}
    \item \textbf{Consolidation.} Training on \(S_m\) decreases the loss primarily by
    optimizing the sparse representation of the already active modes
    \(1,\ldots,m\), and approaches a neighborhood of the sparse rank-\(m\) saddle
    \(\mathcal M_m\).

    \item \textbf{Stalling on the old support.} Near \(\mathcal M_m\), the restricted
    gradient on \(S_m\) has only an \(\alpha_m\)-fraction of the escape direction toward
    mode \(m+1\). Hence training only on \(S_m\) leaves the trajectory near the
    rank-\(m\) sparse saddle.

    \item \textbf{Sparse escape by rewiring.} Training on
    \[
        A_m=S_m\cup G_m
    \]
    recovers a \(1-\varepsilon_m^2\) fraction of the dense escape-direction energy:
    \[
        \|P_{A_m}e_{m+1}(\theta)\|_2^2
        \ge
        (1-\varepsilon_m^2)\|e_{m+1}(\theta)\|_2^2.
    \]
    Therefore the enlarged sparse phase has a nontrivial component in the unstable
    direction that increases the learned strength of singular mode \(m+1\).

    \item \textbf{Support transition.} After the growth phase, magnitude pruning
    preserves the large coordinates created by learning mode \(m+1\). Hence the next
    support \(S_{m+1}\) represents \(M_{m+1}\) up to error \(\delta_{m+1}\).
\end{enumerate}

Consequently, the alternating sparse trajectory shadows the dense saddle-to-saddle
sequence
\[
    \mathcal M_0
    \rightarrow
    \mathcal M_1
    \rightarrow
    \cdots
    \rightarrow
    \mathcal M_r,
\]
up to perturbation errors controlled by $\delta_m$, $\eta_m$, $\alpha_m$, $\varepsilon_m$.
\end{theorem}

\begin{proof}
We prove the claim phase by phase.

\paragraph{Step 1: Consolidation on \(S_m\).}
By Assumption~\ref{ass:sparse-current-mode}, the support \(S_m\) contains a sparse
parameterization of the first \(m\) singular modes. Therefore the restricted flow
\[
    \dot \theta=-P_{S_m}g(\theta)
\]
can reduce the component of the loss associated with the currently active modes.
Near \(M_m\), the residual is
\[
    E_m(\theta)
    =
    M-W_2W_1
    \approx
    \sum_{i>m}s_i u_i v_i^\top.
\]
By Assumption~\ref{ass:residual-dominance}, the dominant remaining residual is the
\((m+1)\)-st singular component.

\paragraph{Step 2: The old support lacks the next escape direction.}
Let \(e_{m+1}(\theta)\) be the dense gradient component that increases the learned
strength of mode \(m+1\). Assumption~\ref{ass:old-support-stall} gives
\[
    \|P_{S_m}e_{m+1}(\theta)\|_2
    \le
    \alpha_m\|e_{m+1}(\theta)\|_2.
\]
Thus the old support contains little of the unstable direction needed to leave the
rank-\(m\) saddle. Training on \(S_m\) can refine the sparse representation of the
current modes, but it cannot efficiently activate mode \(m+1\). This establishes the
stalling behavior.

\paragraph{Step 3: The growth support recovers the escape direction.}
The growth set is chosen as
\[
    G_m
    =
    \operatorname{Top}_q(g(\theta_m)_{S_m^c}).
\]
Since \(G_m\) contains the largest gradient coordinates outside the old support, and
since the next-mode escape component dominates the residual gradient by
Assumption~\ref{ass:residual-dominance}, Assumption~\ref{ass:sparse-escape} implies
\[
    \|(I-P_{A_m})e_{m+1}(\theta)\|_2
    \le
    \varepsilon_m \|e_{m+1}(\theta)\|_2.
\]
Therefore
\[
    \|P_{A_m}e_{m+1}(\theta)\|_2^2
    =
    \|e_{m+1}(\theta)\|_2^2
    -
    \|(I-P_{A_m})e_{m+1}(\theta)\|_2^2
    \ge
    (1-\varepsilon_m^2)\|e_{m+1}(\theta)\|_2^2.
\]
Thus the enlarged support \(A_m\) contains a substantial component of the unstable
direction leading from \(\mathcal M_m\) toward \(\mathcal M_{m+1}\).

\paragraph{Step 4: Sparse escape increases the next mode.}
Let \(q_{m+1}(\theta)\) denote the learned strength of mode \(m+1\), for example
\[
    q_{m+1}(\theta)
    =
    u_{m+1}^\top W_2W_1 v_{m+1}.
\]
Under dense gradient flow, the derivative of \(q_{m+1}\) is positive near
\(\mathcal M_m\) because the residual is dominated by
\(s_{m+1}u_{m+1}v_{m+1}^\top\). Under sparse-growth flow,
\[
    \dot \theta=-P_{A_m}g(\theta).
\]
The change in the mode strength is
\[
    \frac{d}{dt}q_{m+1}(\theta)
    =
    -\langle \nabla q_{m+1}(\theta),P_{A_m}g(\theta)\rangle.
\]
Since \(P_{A_m}g(\theta)\) captures a \(1-\varepsilon_m^2\) fraction of the escape
direction, the sparse-growth phase preserves the positive mode-growth velocity up to
an \(O(\varepsilon_m)\) perturbation. Hence the trajectory moves away from
\(\mathcal M_m\) in the direction of \(\mathcal M_{m+1}\).

\paragraph{Step 5: Magnitude pruning consolidates the new mode.}
During the sparse-growth phase, coordinates that contribute to mode \(m+1\) receive
coherent updates and grow in magnitude. Coordinates not aligned with the first
\(m+1\) modes either receive smaller updates or are not repeatedly reinforced. Hence
magnitude pruning after the growth phase preserves a support \(S_{m+1}\) that
represents
\[
    M_{m+1}
    =
    \sum_{i=1}^{m+1}s_i u_i v_i^\top
\]
up to error \(\delta_{m+1}\), by Assumption~\ref{ass:sparse-current-mode} applied at
the next phase.

Repeating this argument for \(m=0,\ldots,r-1\) gives the claimed sparse
saddle-to-saddle sequence.
\end{proof}

\section{Training details}\label{appendix : training details}

Tables~\ref{tab:training-details}
and~\ref{tab:training-details-acdc_arc} provide the default
hyperparameters for ResNet18/CIFAR100 and DeiT-Small/ImageNet.
The CIFAR100 experiments use cross-entropy loss.
The DeiT-Small experiments use label smoothing with confidence
$0.9$ and CutMix augmentation \citep{yun2019cutmix}.
These experiments use three seeds (1, 2, and 3).
For all vision experiments conducted by us, we report means
and 95\% normal-approximation confidence intervals over the three
seeds.
For all ARC experiments, masked parameters are restored to
their values saved at the end of the preceding dense or
rewiring phase (Stale init), or reset to zero (Zero init),
before computing the gradients used to select $G_t$.
The implementation is based on TurboPrune
\citep{Nelaturu_TurboPrune_High-Speed_Distributed},
and each run uses four NVIDIA A100 GPUs with 40GB memory.
The additional ResNet50 and language-model settings are
described below.

\begin{table}[ht!]
\centering
\caption{
Default training hyperparameters for ResNet18/CIFAR100 and
DeiT-Small/ImageNet.
}
\label{tab:training-details}
\begin{tabular}{lll}
\toprule
\textbf{Configuration} & \textbf{CIFAR100 Experiment} & \textbf{ImageNet Experiment} \\ 
\midrule
\textbf{Dataset}    & CIFAR100    & ImageNet \\ 
\textbf{Model}      & ResNet18    & DeiT-small \\ 
\textbf{LR}         & $0.1$       & $0.05$ \\ 
\textbf{WD}         & $1e-4$      & $1e-1$ \\ 
\textbf{Epochs}     & $150$       & $300$ \\ 
\textbf{Batch Size} & $128$       & $256$ \\ 
\textbf{Optimizer}      & SGD         & AdamW \\ 
\textbf{$\beta_1$}      & --         & 0.9 \\ 
\textbf{$\beta_2$}      & --         & 0.999 \\ 
\textbf{Schedule}   & Triangular  & Triangular \\ 
\bottomrule
\end{tabular}
\end{table}

\begin{table}[ht!]
\centering
\caption{
Default AC/DC, ARC, and RigL settings for ResNet18/CIFAR100
and DeiT-Small/ImageNet. Shortened-phase configurations are
specified separately in the corresponding experiments.
}
\label{tab:training-details-acdc_arc}
\begin{tabular}{lll}
\toprule
Configuration & \textbf{CIFAR100 Experiment} & \textbf{ImageNet Experiment} \\ 
\midrule
ACDC/ARC Dense warmup epochs   & 15  & 15 \\ 
ACDC/ARC Dense epochs   & 10  & 15 \\ 
ACDC/ARC Sparse epochs   & 10  & 15 \\ 
ACDC/ARC Sparse cooldown epochs   & 15  & 15 \\
\midrule
ARC - Batches to compute $G_t$ & 5 & 5 \\
\midrule
RigL - Dense gradient accumulation steps & 100 & 10 \\
RigL - Rewiring steps & 100 & 1000 \\
\bottomrule
\end{tabular}
\end{table}

\subsection{FLOP accounting}
\label{appendix:flop_accounting}

We count one multiply--accumulate as two FLOPs.
Vision training costs are calculated layer by layer from the
active support in each phase and accumulated over the training
schedule, including the gradient probes used for support selection.
For CIFAR100 and DeiT-Small, we report epoch-averaged training
FLOPs per image:
\[
F_{\mathrm{avg}}
=
\frac{F_{\mathrm{total}}}
     {E N_{\mathrm{train}}},
\]
where $E$ is the number of training epochs and
$N_{\mathrm{train}}$ is the number of training images.
For ResNet50, Table~\ref{tab:resnet50_imagenet_95} reports total
training FLOPs relative to the dense-training baseline.
Inference FLOPs are calculated separately from the final support
for one input image.
Global sparsity can produce different layerwise allocations,
so equal parameter sparsity need not imply equal inference FLOPs.

Support selection requires gradient probes over the eligible
parameters, but is performed only at the beginning of each
rewiring phase.
The selected support remains fixed throughout that phase, so
the selection cost is amortized over the subsequent training
updates rather than incurred at every optimization step.

\subsection{ResNet50 on ImageNet}
\label{appendix:resnet50_training}

We follow the ResNet50/ImageNet training recipe from the
AC/DC paper \citep{peste2021acdc}.
The eligible weights are those in convolutional and fully
connected layers, excluding the first convolution and the final
classification head.
For the unstructured experiments, sparse phases use 95\%
target sparsity.
The 100-epoch ARC schedule consists of five dense warmup epochs,
eight cycles of five sparse epochs followed by five rewiring
epochs, and fifteen sparse cooldown epochs.

Each rewiring phase has 30.65\% density across eligible weights:
the inherited 5\% support plus an additional 25.65 percentage
points selected using off-support gradient magnitudes.
Support selection uses five probe batches at each of the eight
rewiring boundaries.
The cost of all forty probe batches is included in the reported
training FLOPs.
ARC--Stale and ARC--Zero use $0.7289$ and $0.7634$ EFLOPs,
respectively, over the complete training run.
Their final inference costs are $0.75437$ and $0.8849$ GFLOPs
per image, respectively.
Both ARC variants are evaluated over three seeds.

\subsection{Language-model training and evaluation}
\label{appendix:language_training}

\paragraph{Shared training recipe.}
We train GPT-2 124M \citep{radford2019language} from scratch
using nanoGPT \citep{karpathy_nanogpt}.
The model has twelve Transformer blocks, twelve attention heads,
hidden dimension 768, tied input and output embeddings,
and no biases or dropout.
All methods use seed 1337 and the same base training recipe.
The context length is 1,024 tokens and the effective batch size
is 480 sequences, or 491,520 tokens per optimizer update.
Training follows a 600,000-update schedule using AdamW with
$(\beta_1,\beta_2)=(0.9,0.95)$, weight decay $0.1$ on matrix
parameters, gradient clipping at $1.0$, and BF16 mixed precision.
The learning rate increases linearly to $6\times10^{-4}$ over
2,000 updates, then follows cosine decay to $6\times10^{-5}$.
It is not restarted at phase boundaries.

\paragraph{Data.}
We use OpenWebText \citep{Gokaslan2019OpenWeb} with the nanoGPT
preprocessing pipeline.
A document-level split holds out $0.05\%$ of the corpus for
validation using split seed 2357.
Documents are tokenized using the GPT-2 byte-pair encoding
tokenizer, separated by end-of-text tokens, and concatenated.
Training samples random 1,024-token windows with replacement.

\paragraph{Sparsity and phase schedules.}
Pruning applies to the 48 attention and MLP weight matrices,
containing 84,934,656 parameters.
Embeddings, the tied output head, and normalization parameters
remain dense.
The reported 80\% sparsity refers to these eligible matrices.
AC/DC and ARC use 120,000 dense warmup updates, three cycles
of 60,000 sparse and 60,000 dense or rewiring updates, and
120,000 sparse cooldown updates.
At each sparse-phase boundary, magnitude pruning is applied
globally across eligible weights.

For ARC, each rewiring phase reactivates the highest-gradient
25\% of masked weights within each layer, giving approximately
40\% overall density across eligible matrices.
Before computing the selection gradients, masked weights are
restored to their stale values or reset to zero according to
the variant.
Both ARC variants average gradients over 100 probe mini-batches.
The inherited and selected supports remain fixed throughout
the rewiring phase, and weights outside this support are
kept at zero.

RigL starts from an 80\%-sparse random mask with ERK layerwise
allocation.
It rewires every 100 updates until update 450,000, using a
cosine-decayed rewiring fraction starting at $0.3$.
It removes the smallest-magnitude active weights and activates
inactive weights with the largest gradient magnitudes.
Newly activated weights and their Adam moments are initialized
to zero.
The mask remains fixed for the final 150,000 updates.

\paragraph{Evaluation.}
All reported metrics use the end-of-training checkpoint,
without checkpoint selection by validation performance.
OpenWebText validation loss and perplexity are evaluated on
4,330 non-overlapping windows of 1,024 tokens.
Zero-shot evaluation uses lm-evaluation-harness version 0.4.3
\citep{eval-harness} through a nanoGPT wrapper.
The wrapper excludes padded vocabulary logits before scoring.
We evaluate the full validation sets for HellaSwag, PIQA,
and WinoGrande, and the full test sets for OpenBookQA,
ARC-Easy, ARC-Challenge, and LAMBADA.
LAMBADA uses the \texttt{lambada\_openai} task.
The reported average is the arithmetic mean of the seven
task scores listed in Table~\ref{tab:gpt2_zero_shot_tasks}.

\section{Additional experiments}

\subsection{Small weights receive large gradients}
The observation that while sparse training causes parameter collapse, dense phase performs rewiring to recover relevant parameters can be observed across the layers of models (\Cref{fig:last_sign_flips_0.1_10D_10S_CIFAR100_multiple_layers}) and even in AC/DC configurations with smaller dense phases (\Cref{fig:last_sign_flips_0.1_3D_17S_CIFAR100_multiple_layers}).

\begin{figure*}[h]
    \centering
    \begin{subfigure}[b]{0.7\textwidth}
        \centering
        \includegraphics[width=\linewidth]{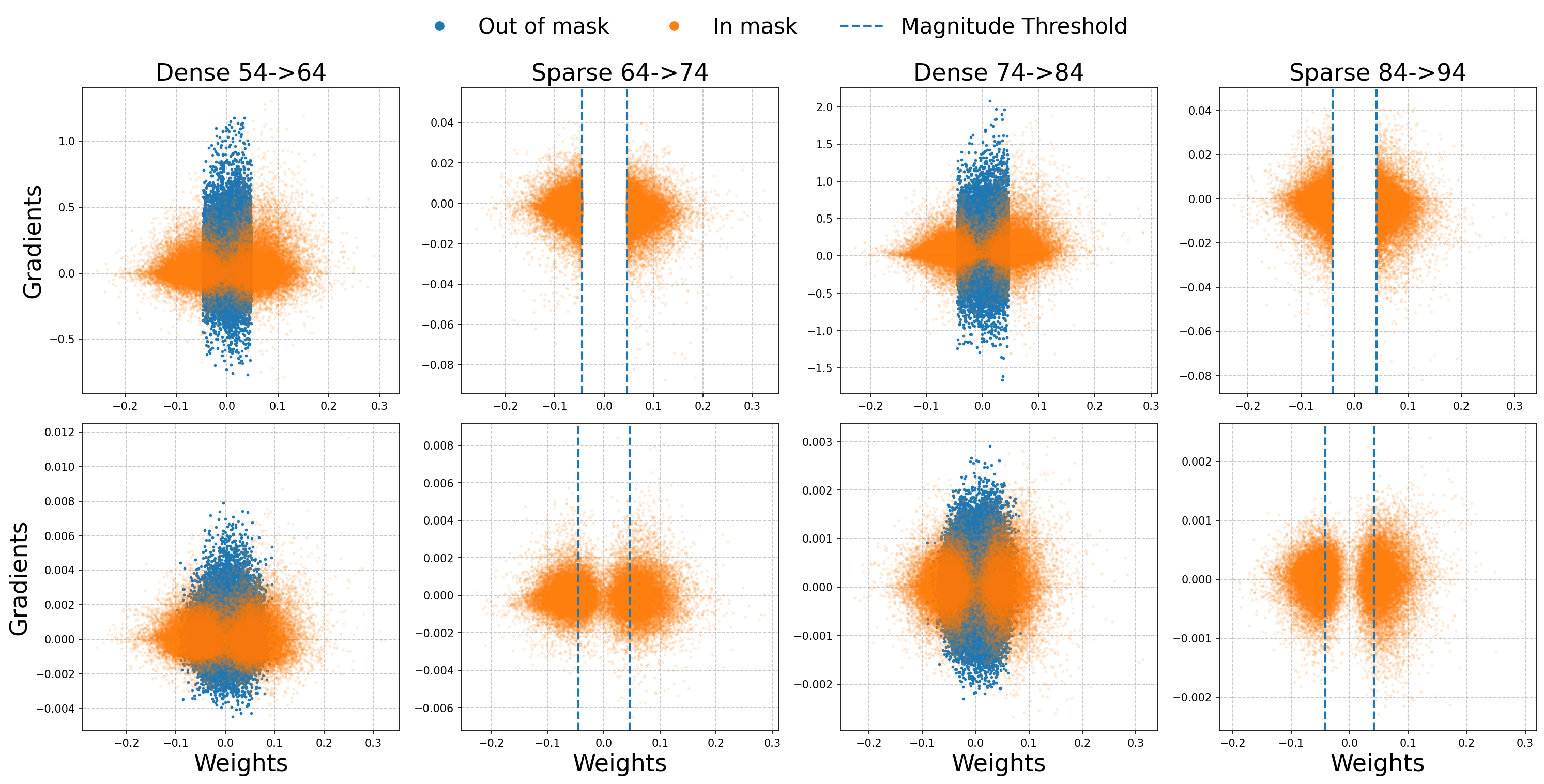}
        \subcaption{\textit{layer2.0.conv2} of ResNet18}
        \label{fig:last_sign_flips_0.1_10D_10S_CIFAR100_layer2}
    \end{subfigure}
    
    \begin{subfigure}[b]{0.7\textwidth}
        \centering
        \includegraphics[width=\linewidth]{Images_new/GradsVsWeights/Density-0.1-Dense-10-Sparse-10-StaleInit-layer4.0.conv2-phase-range-4-8-Seed-1.png}
        \subcaption{\textit{layer4.0.conv2} of ResNet18}
        \label{fig:last_sign_flips_0.1_10D_10S_CIFAR100_layer4}
    \end{subfigure}
    \caption{Gradient vs parameter scatter plots of parameters of layers of a 90\% sparse ResNet18 model on CIFAR100. The plots are shown for the 3rd and 4th dense and sparse phases of the 10-Dense, 10-Sparse schedule of ACDC with \textit{Stale Init}. (top) plots are at the beginning of a phase and (bottom) plots are at the end of a phase.} 
    \label{fig:last_sign_flips_0.1_10D_10S_CIFAR100_multiple_layers}
\end{figure*}

\begin{figure*}[h]
    \centering
    \begin{subfigure}[b]{0.7\textwidth}
        \centering
        \includegraphics[width=\linewidth]{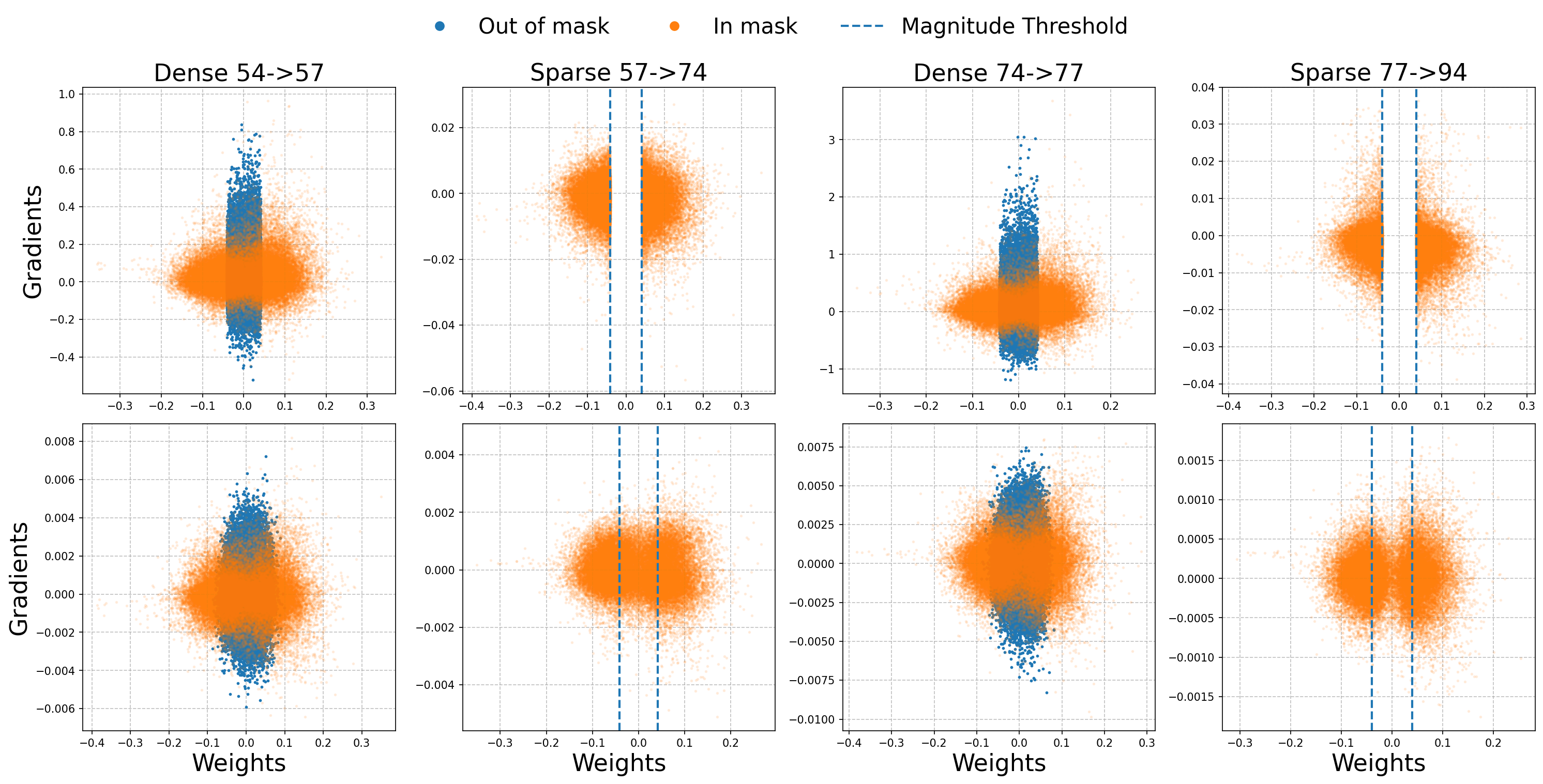}
        \subcaption{\textit{layer2.0.conv2} of ResNet18}
        \label{fig:last_sign_flips_0.1_3D_17S_CIFAR100_layer2}
    \end{subfigure}
    
    \begin{subfigure}[b]{0.7\textwidth}
        \centering
        \includegraphics[width=\linewidth]{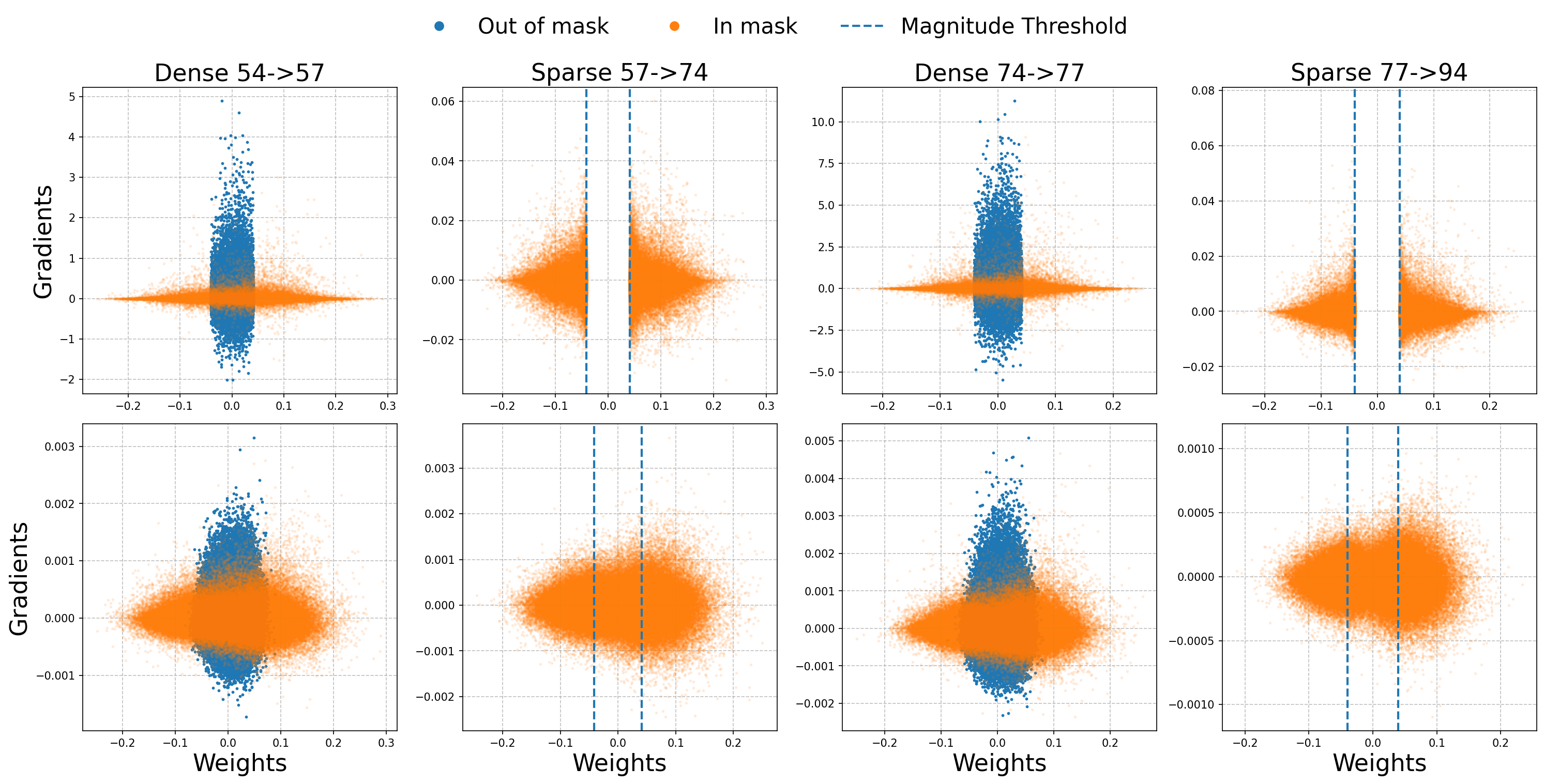}
        \subcaption{\textit{layer4.0.conv2} of ResNet18}
        \label{fig:last_sign_flips_0.1_3D_17S_CIFAR100_layer4}
    \end{subfigure}
    
    \caption{Gradient vs parameter scatter plots of parameters of layers of a 90\% sparse ResNet18 model on CIFAR100. The plots are shown for the 3rd and 4th dense and sparse phases of the 3-Dense, 17-Sparse schedule of ACDC with \textit{Stale Init}. (top) plots are at the beginning of a phase and (bottom) plots are at the end of a phase.} 
    \label{fig:last_sign_flips_0.1_3D_17S_CIFAR100_multiple_layers}
\end{figure*}

\subsection{Replacing the dense phase: gradient-selected versus random support}
Figure~\ref{fig:pareto-dense-random-c100} evaluates replacing
AC/DC's dense phases with training restricted to a fixed random
subset of parameters on ResNet18/CIFAR100.
Random selection avoids the dense-gradient computation required
to identify the rewiring support, but yields lower accuracy
than selecting parameters with the largest gradient magnitudes.
This comparison supports using gradient information to guide
support selection during rewiring.

\begin{figure}[ht!]
    \centering
    \includegraphics[width=0.6\linewidth]{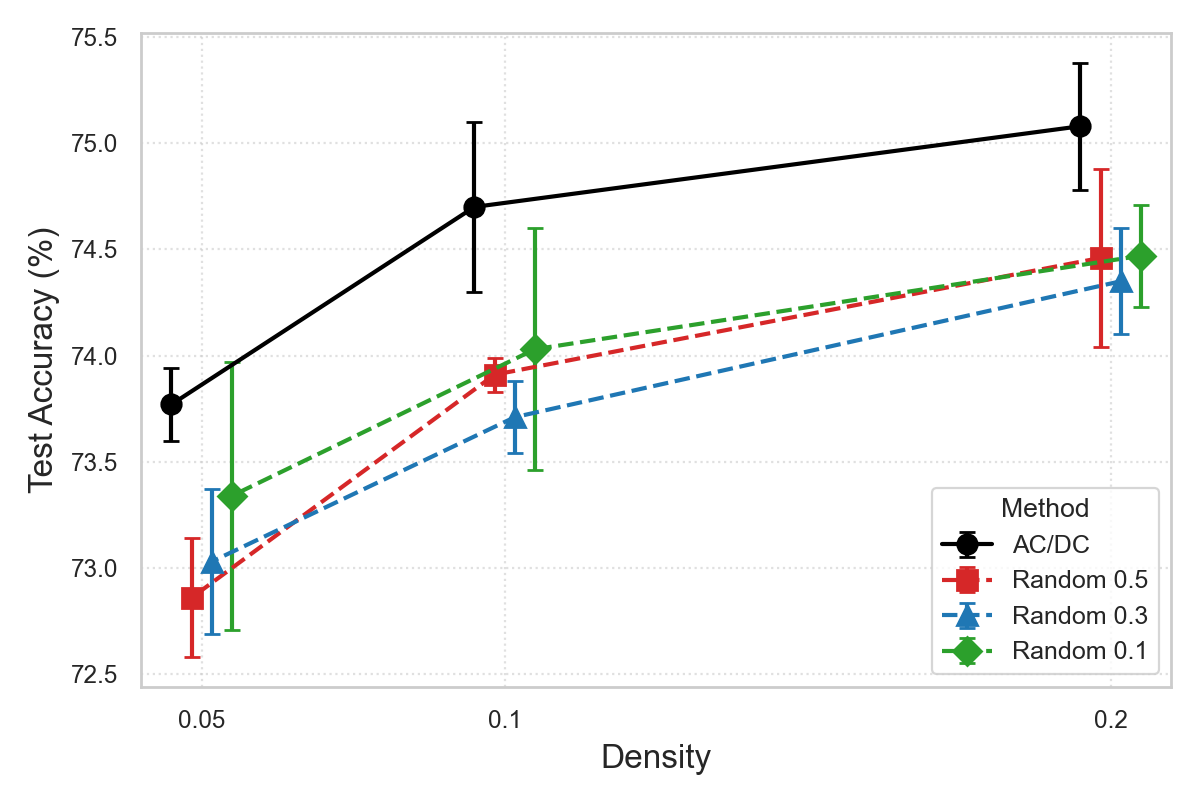}
    \caption{Replacing dense phases in AC/DC with updating only a random subset of parameters for CIFAR100 trained on ResNet18.}
    \label{fig:pareto-dense-random-c100}
\end{figure}

\subsection{Language-model evaluation results}
\label{appendix:language_results}

Table~\ref{tab:gpt2_zero_shot_tasks} reports the complete
zero-shot evaluation.
We use normalized accuracy (\texttt{acc\_norm}) for OpenBookQA,
ARC-Challenge, ARC-Easy, PIQA, and HellaSwag, and accuracy
(\texttt{acc}) for WinoGrande and LAMBADA.
ARC--Zero achieves the highest average among the sparse methods
and improves over AC/DC on six of the seven tasks.

For context, the dense nanoGPT reproduction
\citep{karpathy_nanogpt} reports an OpenWebText validation loss
of approximately $2.85$, corresponding to perplexity $17.3$.
Our evaluation of the released OpenAI GPT-2 checkpoint
\citep{radford2019language} gives validation loss $3.11387$ and
perplexity $22.5081$.
These are external dense reference points rather than matched
training baselines; the OpenAI checkpoint was trained on WebText,
not OpenWebText.

\begin{table}[t]
    \centering
    \caption{
    \textbf{Zero-shot evaluation of GPT-2 124M.}
    Scores are percentages from one training seed (1337).
    Bold indicates the best score among the four sparse methods.
    The final column evaluates the released dense OpenAI GPT-2
    checkpoint as an external reference.
    }
    \label{tab:gpt2_zero_shot_tasks}
    \small
    \setlength{\tabcolsep}{4pt}
    \begin{tabular*}{\linewidth}
        {@{\extracolsep{\fill}}lrrrrr@{}}
        \toprule
        Task & AC/DC & ARC--Stale & ARC--Zero & RigL
             & \shortstack{Dense GPT-2\\reference} \\
        \midrule
        OpenBookQA \citep{mihaylov2018can}
            & 26.80 & \textbf{27.00} & 26.20 & 25.40 & 26.40 \\
        ARC-Challenge \citep{clark2018think}
            & 23.21 & 23.38 & \textbf{23.55} & 22.10 & 22.10 \\
        ARC-Easy \citep{clark2018think}
            & 38.05 & 37.63 & \textbf{39.23} & 37.24 & 40.24 \\
        WinoGrande \citep{sakaguchi2021winogrande}
            & 50.28 & 48.86 & \textbf{52.64} & 50.05 & 51.85 \\
        PIQA \citep{bisk2020piqa}
            & 59.09 & 59.90 & \textbf{61.59} & 60.34 & 62.40 \\
        HellaSwag \citep{zellers2019hellaswag}
            & 29.24 & 29.67 & \textbf{29.76} & 28.75 & 31.50 \\
        LAMBADA \citep{paperno2016lambada}
            & 27.23 & 28.00 & \textbf{28.06} & 25.83 & 30.74 \\
        \midrule
        Average
            & 36.27 & 36.35 & \textbf{37.29} & 35.67 & 37.89 \\
        \bottomrule
    \end{tabular*}
\end{table}

\subsection{2:4 semi-structured sparsity}
\label{appendix:semi_structured}

We evaluate a 2:4 semi-structured extension on ResNet50/ImageNet
using the same base training recipe, phase schedule, and eligible
layers as the unstructured ResNet50 experiments in
Appendix~\ref{appendix:resnet50_training}.
The first convolution and final classification head remain dense.

During sparse phases, eligible convolutional weights follow a
2:4 pattern.
For each output channel and kernel position $(r,s)$, groups
consist of four consecutive input channels.
Magnitude pruning retains two weights in each group.
This grouping follows NVIDIA's documented convolutional
structured-sparsity convention
\citep{nvidia_tensorrt_sparsity}.

At the beginning of each ARC rewiring phase, we use five probe
mini-batches to score output channels by their off-mask gradient
energy, separately within each layer.
The highest-scoring half of the output channels are temporarily
made dense, while the remaining channels retain their existing
2:4 support.
This gives 75\% parameter density in eligible layers during
rewiring, rather than fully dense decompression.
The selected channel partition remains fixed throughout the
phase.
Before the following sparse phase, the temporarily dense
channels are pruned back to 2:4 using weight magnitude.

Table~\ref{tab:resnet50_semi_structured} reports results over
three seeds with 95\% normal-approximation confidence intervals.
ARC--Stale and ARC--Zero are within 0.33 and 0.35 percentage
points of their corresponding AC/DC variants.
Thus, targeted support expansion also retains accuracy close
to AC/DC under a semi-structured sparsity constraint.
These experiments evaluate accuracy; they do not establish
measured wall-clock acceleration.

\begin{table}[t]
    \centering
    \caption{
    \textbf{2:4 semi-structured ResNet50 on ImageNet.}
    Top-1 accuracy (\%) is reported as the mean and 95\%
    normal-approximation confidence interval over three seeds.
    ARC uses 75\% density in eligible layers during rewiring;
    AC/DC uses fully dense decompression.
    }
    \label{tab:resnet50_semi_structured}
    \small
    \begin{tabular}{@{}lc@{}}
        \toprule
        Method & Top-1 accuracy (\%) \\
        \midrule
        AC/DC--Stale & $76.29 \pm 0.06$ \\
        ARC--Stale   & $75.96 \pm 0.14$ \\
        \midrule
        AC/DC--Zero  & $76.41 \pm 0.07$ \\
        ARC--Zero    & $76.06 \pm 0.07$ \\
        \bottomrule
    \end{tabular}
\end{table}

\begin{figure*}[ht!]
    \centering
    \begin{subfigure}{0.9\textwidth}
        \centering
        \includegraphics[width=\textwidth]{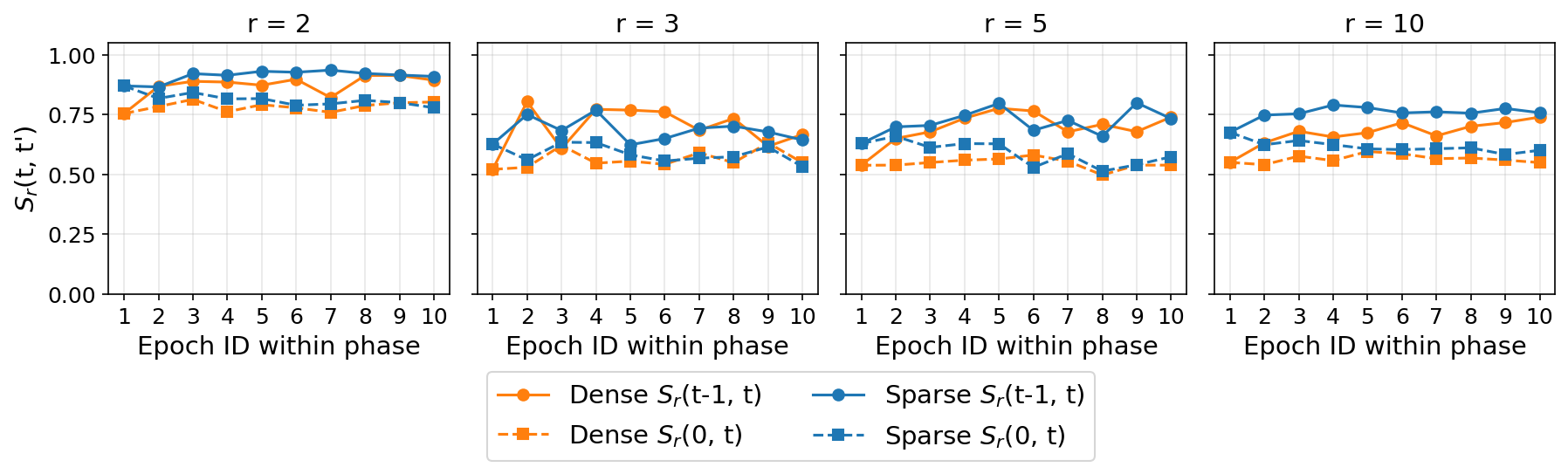}
        \caption{First cycle (Sparse:24-34, Dense:34-44)}
        \label{fig:activation_subspace_phase0_all}
    \end{subfigure}
    
    \begin{subfigure}{0.9\textwidth}
        \centering
        \includegraphics[width=\textwidth]{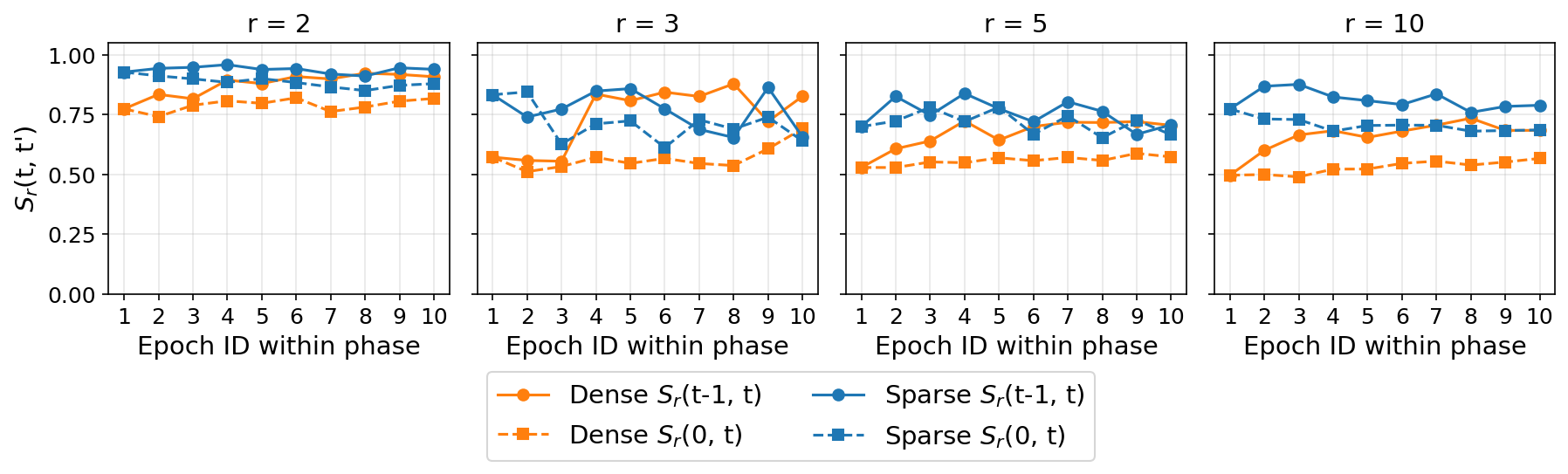}
        \caption{Second cycle ((Sparse:44-54, Dense:54-64))}
        \label{fig:activation_subspace_phase1_all}
    \end{subfigure}

    \begin{subfigure}{0.9\textwidth}
        \centering
        \includegraphics[width=\textwidth]{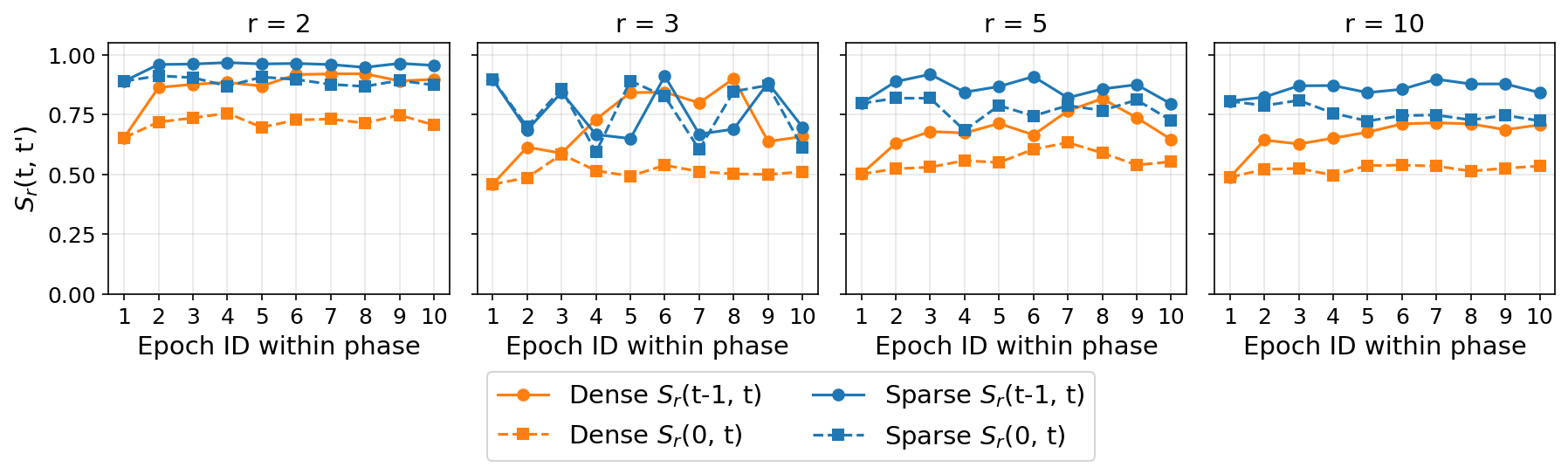}
        \caption{Third cycle ((Sparse:64-74, Dense:74-84))}
        \label{fig:activation_subspace_phase2_all}
    \end{subfigure}

    \begin{subfigure}{0.9\textwidth}
        \centering
        \includegraphics[width=\textwidth]{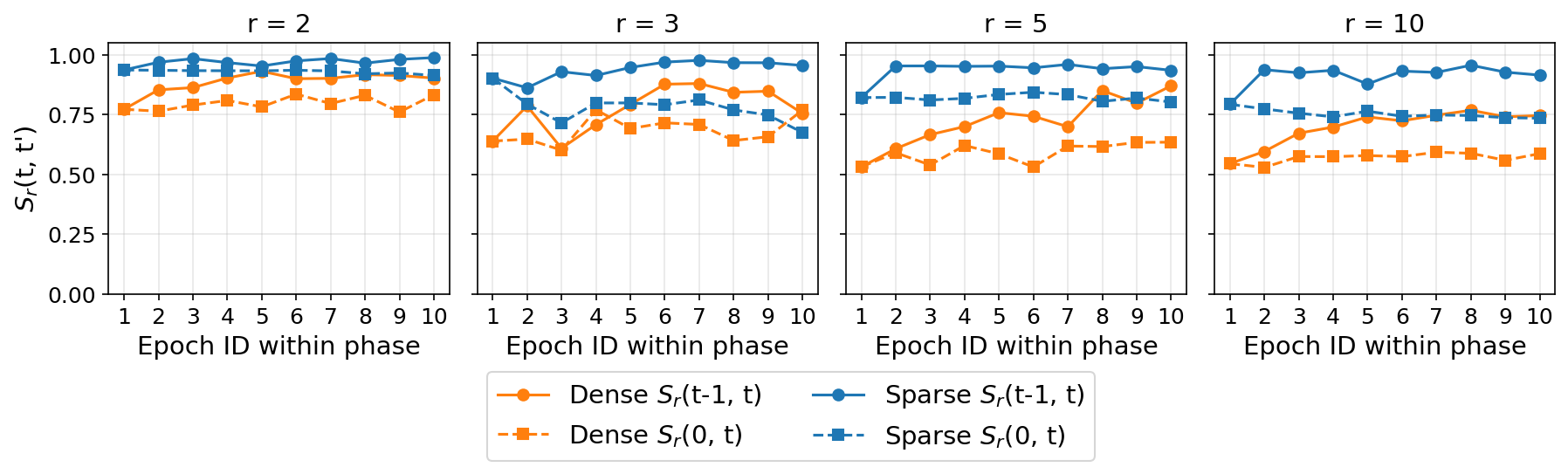}
        \caption{Fourth cycle ((Sparse:84-94, Dense:94-104))}
        \label{fig:activation_subspace_phase3_all}
    \end{subfigure}

    \caption{
\textbf{Activation subspace similarity across the first four AC/DC cycles.}
For \textit{layer3.1.conv2} of a 90\% sparse ResNet18 trained on CIFAR100, we compare the subspaces spanned by the top \(r\in\{2,3,5,10\}\) activation Gram eigenvectors. We report consecutive-checkpoint similarity \(S_r(t-1,t)\) and similarity to the first checkpoint of each phase \(S_r(0,t)\). Sparse phases consistently preserve leading activation subspaces more strongly than dense phases, especially for \(r=5\) and \(r=10\). Dense phases induce larger but smooth drift, consistent with their role as rewiring windows.
}
    \label{fig:activation_subspace_similarity_first_four_cycles}
\end{figure*}

\section{Additional Activation Eigenspace Diagnostics}
\label{appendix:activation_eigenspaces}

This appendix provides the construction and additional results
for the activation-subspace diagnostic in
Section~\ref{sec:activation_subspace_stability}.

\paragraph{Activation Gram matrix.}
For a fixed convolutional layer, let
$Z_t(x)\in\mathbb R^{C\times H\times W}$ denote its activations
for image $x$ at checkpoint $t$, where $C$ is the number of
output channels and $H,W$ are the spatial feature-map dimensions.
Flattening the spatial dimensions gives
$F_t(x)\in\mathbb R^{C\times HW}$.
We aggregate the per-image Gram matrices over the CIFAR100
training set $\mathcal D$ before computing the eigenvectors:
\begin{equation}
\label{eq:activation_channel_gram}
K_t
=
\frac{1}{|\mathcal D|}
\sum_{x\in\mathcal D} F_t(x)F_t(x)^\top
\in\mathbb R^{C\times C}.
\end{equation}
Summing instead of averaging multiplies $K_t$ by a positive
scalar and therefore leaves its eigenspaces unchanged.
Each entry measures the average inner product between two
channels' spatial activation maps.
The leading eigenvectors identify channel combinations with
the greatest activation energy.

\paragraph{Subspace comparison.}
Let $U_t^{(r)}\in\mathbb R^{C\times r}$ contain orthonormal
eigenvectors corresponding to the $r$ largest eigenvalues of $K_t$.
We define
\[
S_r(t_1,t_2)
=
\frac{1}{r}
\left\|
{U_{t_1}^{(r)}}^\top U_{t_2}^{(r)}
\right\|_F^2.
\]
This is the average squared cosine of the principal angles
between the two subspaces.
It lies in $[0,1]$, equals one for coincident subspaces,
and equals zero for orthogonal subspaces.
Comparing subspaces avoids dependence on individual eigenvector
signs or the choice of orthonormal basis within a retained subspace.
We report consecutive-checkpoint similarity $S_r(t-1,t)$ and
similarity to the first checkpoint of each phase, $S_r(0,t)$.
The time index is reset at the beginning of each phase.
We additionally show individual-eigendirection similarity
heatmaps to examine changes within the leading spectrum.

Figure~\ref{fig:activation_subspace_similarity_first_four_cycles} reports this metric for the first four AC/DC cycles of \textit{layer3.1.conv2}. Across cycles, sparse phases preserve leading activation subspaces more strongly than dense phases. This difference is most visible for \(r=5\) and \(r=10\). Later cycles show especially high sparse-phase similarity, consistent with increasing representation stability over training. Dense phases show larger drift from the phase start, but consecutive-epoch similarities remain substantial, indicating gradual representation movement rather than abrupt resets.

\begin{figure*}[h]
    \centering
    \begin{subfigure}{0.9\textwidth}
        \centering
        \includegraphics[width=\textwidth]{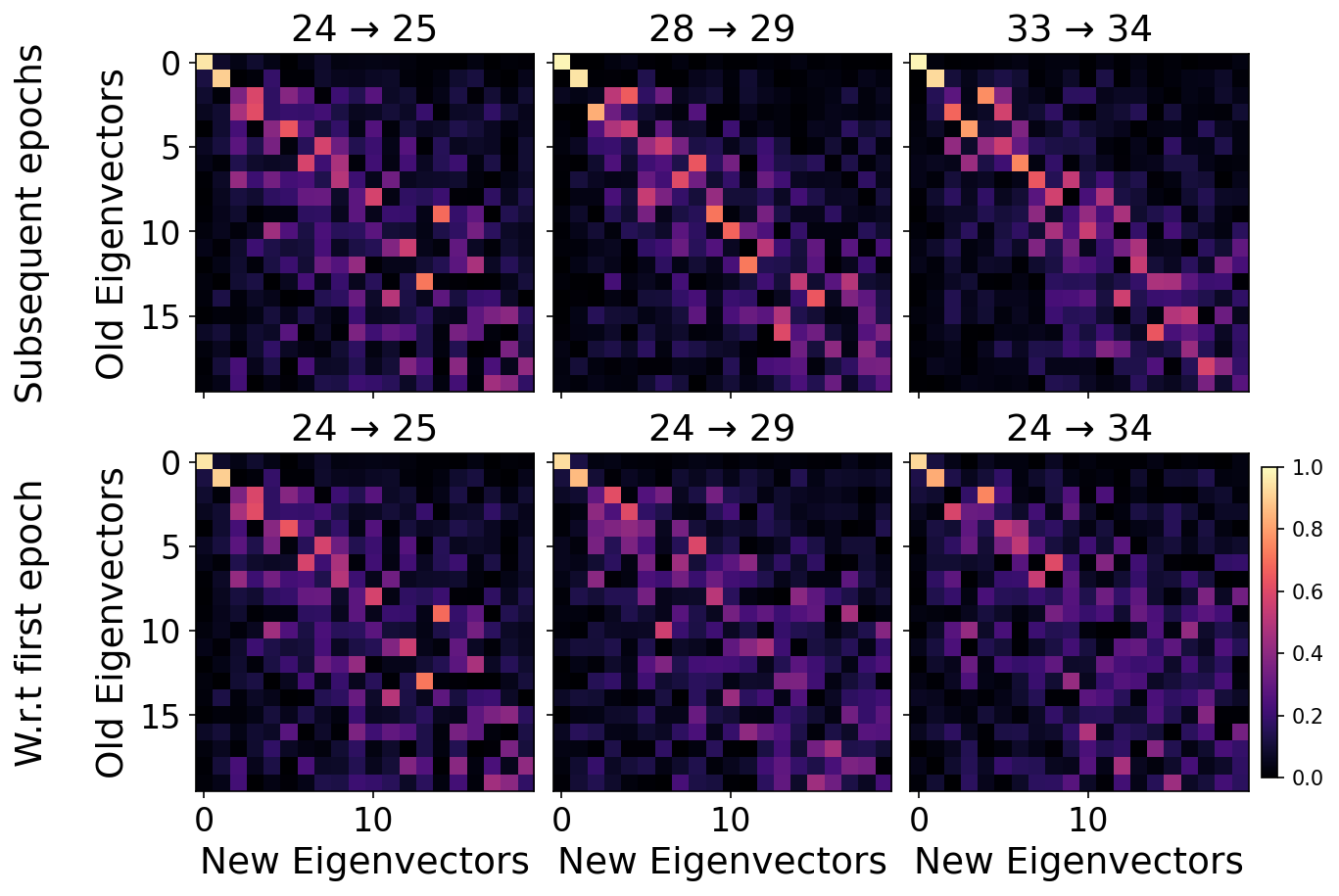}
        \caption{Sparse:24-34}
        \label{fig:activation_heatmaps_phase0_sparse}
    \end{subfigure}
    
    \begin{subfigure}{0.9\textwidth}
        \centering
        \includegraphics[width=\textwidth]{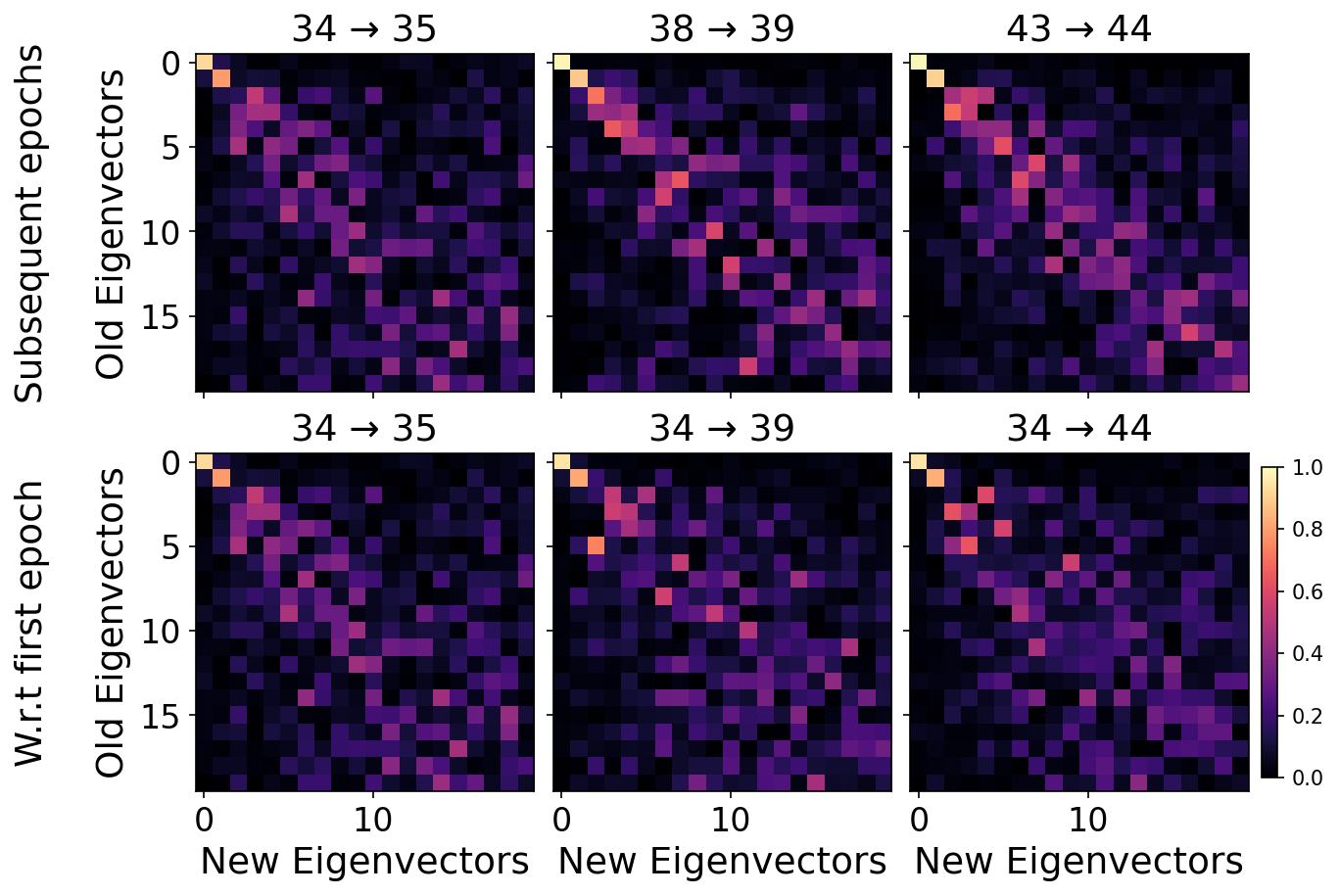}
        \caption{Dense:34-44}
        \label{fig:activation_heatmaps_phase0_dense}
    \end{subfigure}

    \caption{
\textbf{Activation eigendirection similarities in the first AC/DC cycle.}
We compare activation-Gram eigenvectors for
\textit{layer3.1.conv2} of a 90\% sparse ResNet18 on CIFAR100
during the sparse phase (epochs 24--34) and dense phase
(epochs 34--44).
Within each panel, the top row compares consecutive checkpoints
and the bottom row compares checkpoints with the phase start.
Rows and columns index eigenvectors at the earlier and later
checkpoints, respectively.
}
    \label{fig:activation_heatmap_first_cycle}
\end{figure*}

\begin{figure*}[ht!]
    \centering
    \begin{subfigure}{0.9\textwidth}
        \centering
        \includegraphics[width=\textwidth]{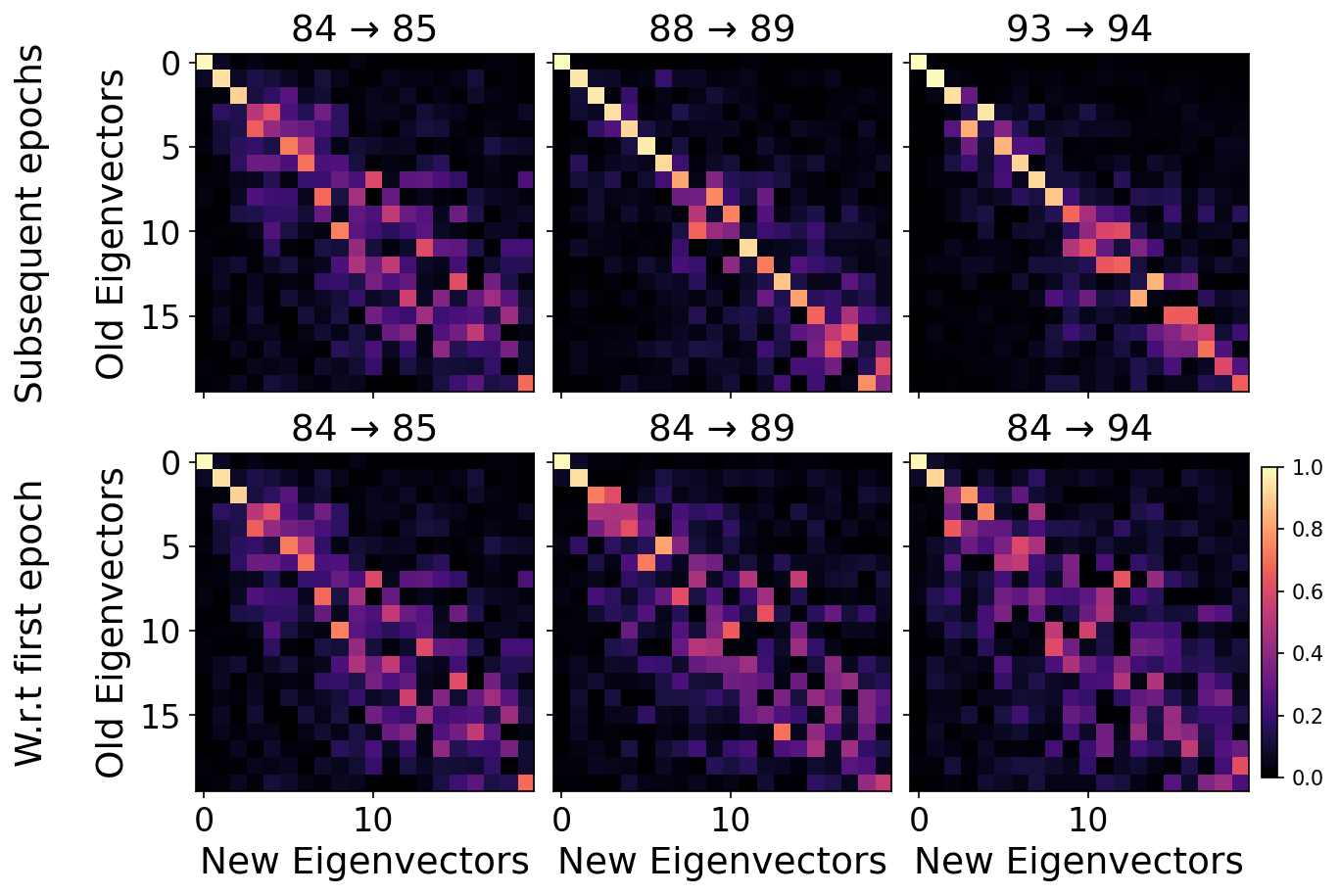}
        \caption{Sparse:84-94}
        \label{fig:activation_heatmap_phase3_sparse}
    \end{subfigure}
    
    \begin{subfigure}{0.9\textwidth}
        \centering
        \includegraphics[width=\textwidth]{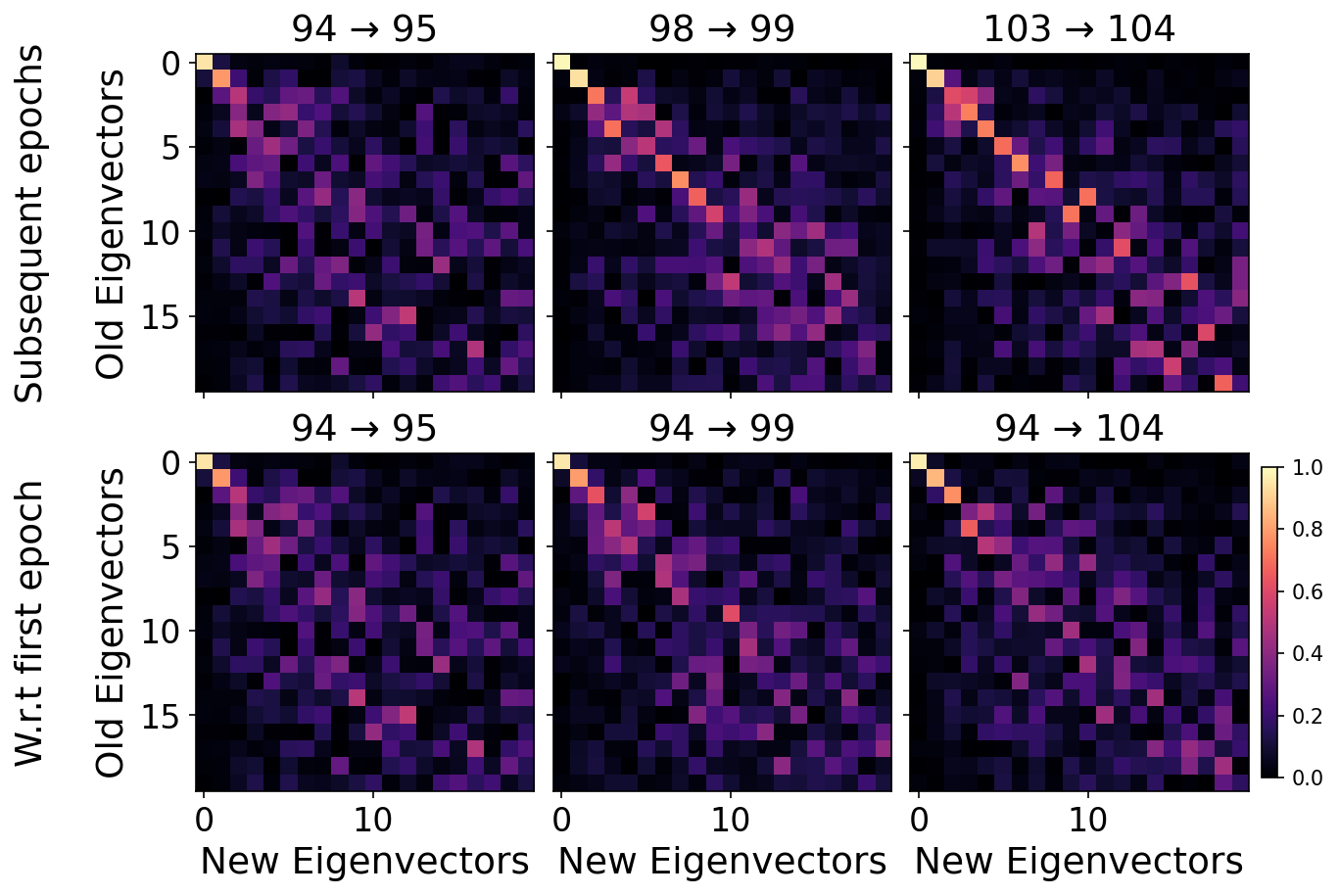}
        \caption{Dense:94-104}
        \label{fig:activation_subspace_phase3_dense}
    \end{subfigure}

    \caption{
\textbf{Activation eigendirection similarities in the fourth AC/DC cycle.}
We compare activation-Gram eigenvectors for
\textit{layer3.1.conv2} of a 90\% sparse ResNet18 on CIFAR100
during the sparse phase (epochs 84--94) and dense phase
(epochs 94--104).
Within each panel, the top row compares consecutive checkpoints
and the bottom row compares checkpoints with the phase start.
Rows and columns index eigenvectors at the earlier and later
checkpoints, respectively.
}
    \label{fig:activation_heatmap_forth_cycle}
\end{figure*}

\begin{figure*}[ht!]
    \centering
    \begin{subfigure}[b]{0.45\textwidth}
        \centering
        \includegraphics[width=\linewidth]{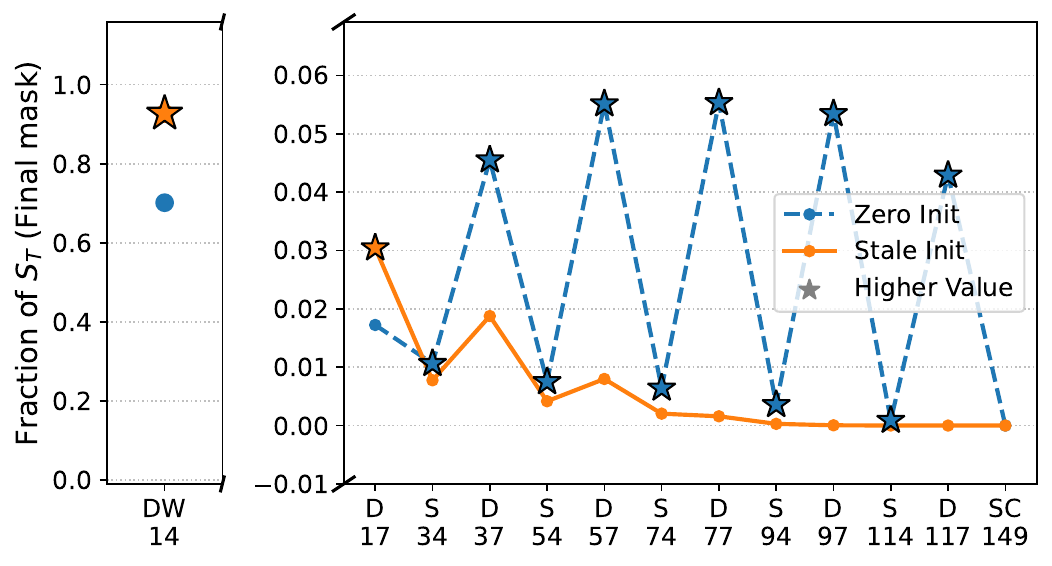}
        \subcaption{}
        \label{fig:last_sign_flips_0.1_3D_17S_CIFAR100_a}
    \end{subfigure}
    \hfill
    \begin{subfigure}[b]{0.45\textwidth}
        \centering
        \includegraphics[width=\linewidth]{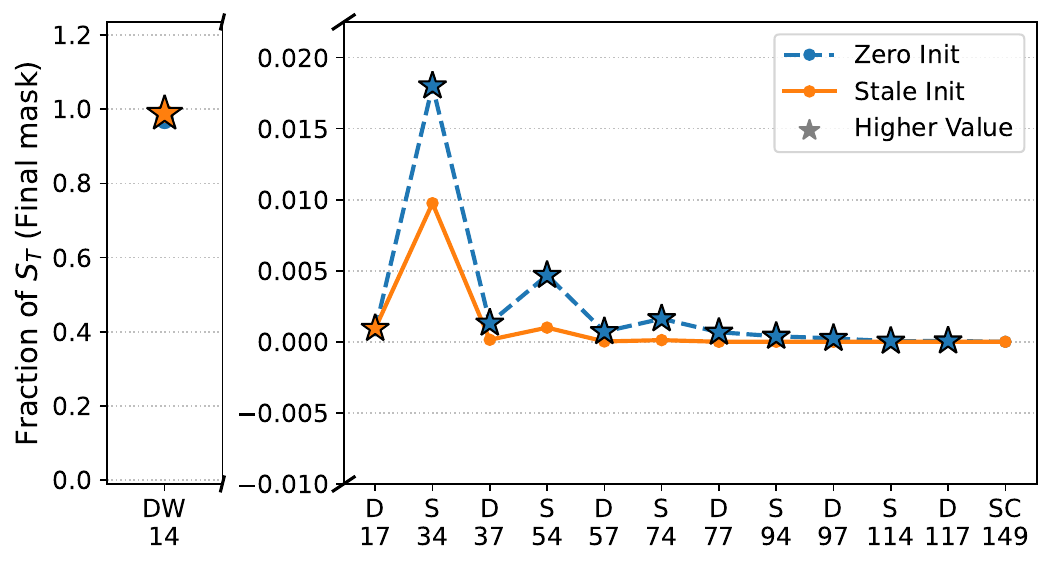}
        \subcaption{}
        \label{fig:last_sign_flips_0.1_3D_17S_CIFAR100_b}
    \end{subfigure}
    
    \caption{The plot shows the fraction of weights of ResNet18 that flip their sign for the last time during a specific phase of AC/DC (3-dense, 17-sparse). Panels compare sign convergence for parameters in (a) the final converged mask and (b) the stable core weights found in the intersection of all masks.}
    \label{fig:last_sign_flips_0.1_3D_17S_CIFAR100}
\end{figure*}

\begin{figure*}[ht!]
    \centering
    \begin{subfigure}[b]{0.45\textwidth}
        \centering
        \includegraphics[width=\linewidth]{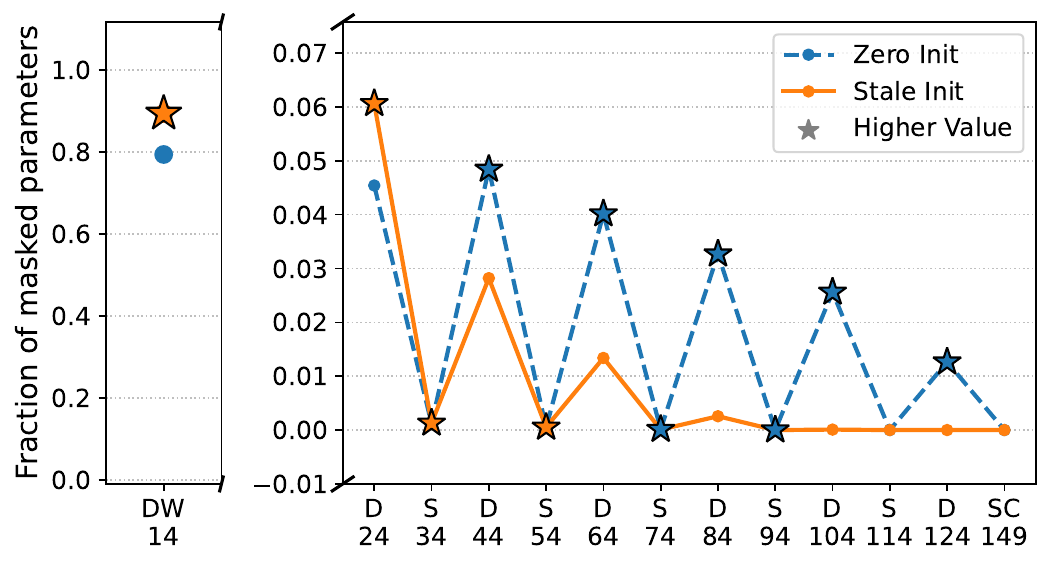}
        \subcaption{\textit{layer2.0.conv2} of ResNet18}
        \label{fig:last_sign_flips_0.1_10D_10S_CIFAR100_layer2.0.conv2}
    \end{subfigure}
    \hfill
    \begin{subfigure}[b]{0.45\textwidth}
        \centering
        \includegraphics[width=\linewidth]{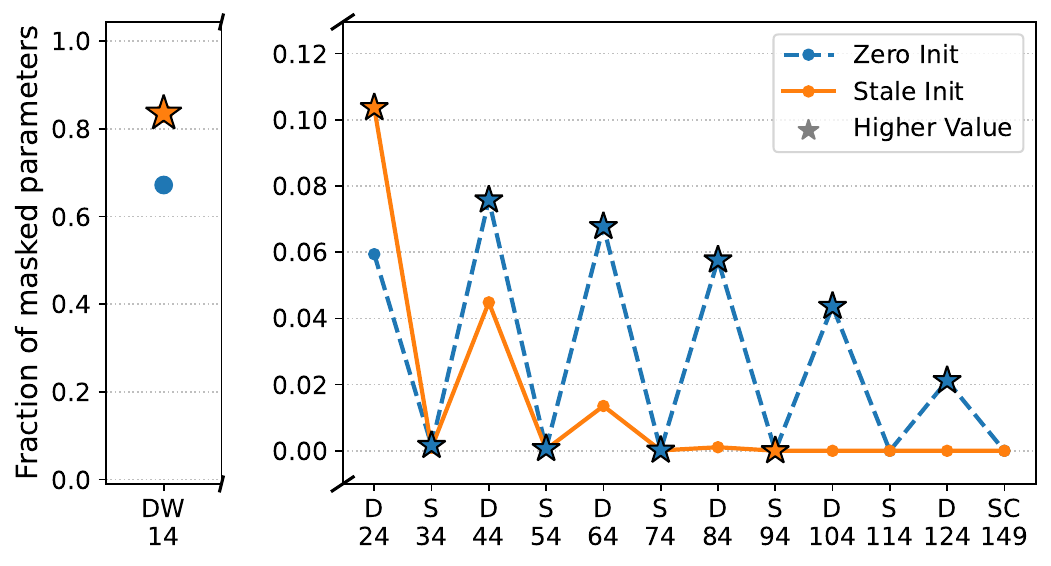}
        \subcaption{\textit{layer4.0.conv2} of ResNet18}
        \label{fig:last_sign_flips_0.1_10D_10S_CIFAR100_layer4.0.conv2}
    \end{subfigure}
    
    \caption{The plot shows the fraction of weights of specific layers of ResNet18 that flip their sign for the last time during a specific phase of AC/DC (10-dense, 10-sparse). Panels compare sign convergence for parameters of layers \textit{layer2.0.conv2}(a) and \textit{layer4.0.conv2} in the final converged mask.}
    \label{fig:last_sign_flips_0.1_10D_10S_CIFAR100_multipl_layers}
\end{figure*}

Figures~\ref{fig:activation_heatmap_first_cycle}, \ref{fig:activation_heatmap_forth_cycle} provide the eigendirection cosine-similarity heatmaps for dense and sparse phases of the first and fourth cycles. These heatmaps show structured alignment among leading eigendirections, with greater mixing among lower-ranked directions. They provide a more detailed view of the subspace trends summarized in the main text.

\section{Final-Mask Sign Flips Occur During Rewiring Windows}\label{appendix : sign flip rewire}

\label{subsec:sign_flip_timing}

Prior work on lottery-ticket-style sparse training shows that
surviving weights are characterized not only by their mask but
also by their signs
\citep{frankle2019lottery,gadhikar2024masks}.
Sign-In to the Lottery further studies sign correction as a
mechanism for improving sparse training
\citep{gadhikar2025signin}.
We use the timing of the last sign flip to identify when
parameters undergo persistent directional corrections.

\paragraph{Counting protocol.}
A sign flip is a reversal between two nonzero parameter values.
During dense phases, each parameter is compared against its
last active nonzero value, carrying this reference across
masked or zero-valued intervals.
During sparse phases, active parameters are compared between
consecutive checkpoints; pruned parameters cannot flip.
Pruning and zero initialization are not themselves counted
as sign reversals.

Each parameter is assigned to the phase containing its last
sign flip.
Parameters that never flip are assigned to the warmup bin,
which therefore represents parameters whose signs have settled
by the end of warmup.
We apply this analysis separately to parameters retained in
the final sparse mask and to parameters active at every
intermediate checkpoint, for both Zero init and Stale init.
For the final-mask plots, phase counts are normalized by
the final mask size.

\paragraph{Additional AC/DC diagnostics.}
The main aggregate comparison is shown in
Figure~\ref{fig:last_sign_flips_compare}a.
Most final-mask parameters settle their signs by the end of
warmup, while later sign corrections occur mainly during
dense phases.
Stale initialization concentrates these corrections in early
dense phases; zero initialization continues to produce
corrections later in training.

Figure~\ref{fig:last_sign_flips_0.1_3D_17S_CIFAR100} compares
the final-mask and always-active parameter sets under shortened
dense phases.
Figure~\ref{fig:last_sign_flips_0.1_10D_10S_CIFAR100_multipl_layers}
shows the corresponding layerwise analysis.
The ability of Zero init to produce meaningful sign flips
later in training persists across layers and shortened-phase
configurations.
Together with its stronger performance in
Figure~\ref{fig:pareto-dense-phases-c100}, these findings
indicate that Zero init provides better directional correction.

\paragraph{Sign correction under ARC.}
Figure~\ref{fig:last_sign_flips_compare}b applies the same final-mask sign-flip diagnostic to ARC.
Late sign corrections occur primarily during ARC rewiring phases rather than sparse phases.
Thus, replacing dense phases with sparse updates on \(A_t=S_t\cup G_t\) preserves the phase-wise correction pattern observed in AC/DC, but compresses it onto a smaller active support.

The comparison between Stale init and Zero init also changes under ARC.
In AC/DC, Stale init concentrates most late sign corrections in early dense phases.
Under ARC, Stale init can maintain nontrivial sign correction across later rewiring phases.
This suggests that restricting updates to high-gradient reintroduced coordinates changes the reactivation dynamics while preserving the correction role of expanded phases.

\newpage

\end{document}